\pdfoutput=1
\documentclass[11pt]{article}

\usepackage{acl}
\usepackage{times}
\usepackage{latexsym}
\usepackage{enumitem}
\usepackage[T1]{fontenc}
\usepackage[utf8]{inputenc}

\usepackage{microtype}

\usepackage{algorithm}
\usepackage{algpseudocode}
\usepackage{amssymb}
\usepackage{amsmath,amsfonts}
\usepackage{algorithm}
\usepackage{algpseudocode} 
\usepackage{booktabs}
\usepackage{multirow}
\usepackage{graphicx}
\usepackage{pifont}
\usepackage{makecell}
\usepackage{svg}
\usepackage[resetstyle]{phfthm}

\title{Improving Conversational Abilities of Quantized Large Language Models via Direct Preference Alignment}
\author{Janghwan Lee${}^{1}$\thanks{\, Equal contribution\quad${}^\dagger$Corresponding author} , Seongmin Park${}^{1}$\footnotemark[1] , Sukjin Hong${}^{1,2}$, Minsoo Kim${}^{1}$, \\ {\bf Du-Seong Chang${}^{2}$} and {\bf Jungwook Choi${}^{1\dagger}$}  \\
        \normalsize{\textsuperscript{1}Hanyang University},
        \normalsize{\textsuperscript{2}KT} \\
        \normalsize{Seoul, Republic of Korea}\\
        \small{\textsuperscript{1}\texttt{\{hwanii0288, skstjdals, sjhong7898, minsoo2333\}@hanyang.ac.kr}} \\
        \small{\textsuperscript{2}\texttt{\{sukjin.hong, dschang\}@kt.com}, \textsuperscript{1$\dagger$}\texttt{choij@hanyang.ac.kr}} \\   
}
\begin{document}
\maketitle
\begin{abstract}

The rapid advancement of large language models (LLMs) has facilitated their transformation into conversational chatbots that can grasp contextual nuances and generate pertinent sentences, closely mirroring human values through advanced techniques such as instruction tuning and reinforcement learning from human feedback (RLHF). However, the computational efficiency required for LLMs, achieved through techniques like post-training quantization (PTQ), presents challenges such as token-flipping that can impair chatbot performance. In response, we propose a novel preference alignment approach, quantization-aware direct preference optimization (QDPO), that aligns quantized LLMs with their full-precision counterparts, improving conversational abilities. Evaluated on two instruction-tuned LLMs in various languages, QDPO demonstrated superior performance in improving conversational abilities compared to established PTQ and knowledge-distillation fine-tuning techniques, marking a significant step forward in the development of efficient and effective conversational LLMs.

% The rapid advancement of large language models (LLMs) has facilitated their transformation into conversational chatbots that can grasp contextual nuances and generate pertinent sentences, closely mirroring human values through advanced techniques such as instruction tuning and reinforcement learning from human feedback (RLHF). However, the computational efficiency required for LLMs, achieved through techniques like post-training quantization (PTQ), presents challenges such as token-flipping that can impair chatbot performance. In response, we introduce a novel preference alignment approach, quantization-aware direct preference optimization (QDPO), designed to synchronize quantized LLMs with their full-precision counterparts, thereby enhancing conversational capabilities. When evaluated on two instruction-tuned LLMs across different languages, QDPO demonstrated superior performance in improving conversational abilities compared to established PTQ and knowledge-distillation fine-tuning techniques, marking a significant step forward in the development of efficient and effective conversational LLMs.

\end{abstract}

\section{Introduction}
\label{sec:introduction}

As large language models (LLMs) advance in understanding the context of language and generating relevant sentences, LLMs are evolving into conversational chatbots that can naturally respond to a wide array of user requests~\cite{gpt4,vicuna2023,geminiteam2023gemini,touvron2023llama2}. Particularly noteworthy is the remarkable ability of LLMs to follow user instructions and align with human values, such as providing helpful and engaging responses through techniques like instruction tuning and reinforcement learning from human feedback (RLHF)~\cite{alpaca,longpre2023flan,chung2022scaling,mukherjee2023orca,ouyang2022training}. These advancements have greatly enhanced the capability to fine-tune pre-trained LLMs for various tasks and user preferences. % Once these models are aligned with human values, they become the preferred choice for users, significantly surpassing the original, unaligned models on which they are based [REFs].
% LLM as a judge?

For the effective implementation of LLM-based chatbots, addressing LLMs' computational complexity is essential. Weight load overhead, a critical bottleneck in LLM deployment, has led to the development of weight quantization techniques like post-training quantization (PTQ). PTQ reduces storage requirements by applying quantization to the weights of trained LLMs, thereby decreasing the necessary bit count for weight data storage ~\cite{frantar2023optq,lin2023awq}. Techniques such as AWQ ~\cite{lin2023awq} address quantization-induced accuracy loss through methods like scaling data distribution and weight updates aimed at preserving accuracy. The effectiveness of these quantization strategies has been measured by task-dependent benchmarks to evaluate model accuracy instead of multifaceted conversational qualities.

% \begin{figure}[t]
% \centerline{\includegraphics[width=\columnwidth]{Figures/fig1.pdf}}
% \caption{Relative performance vs. Baseline}
% \label{fig:fig1}
% \end{figure}
\begin{figure}[t]
\centerline{\includegraphics[width=\columnwidth]{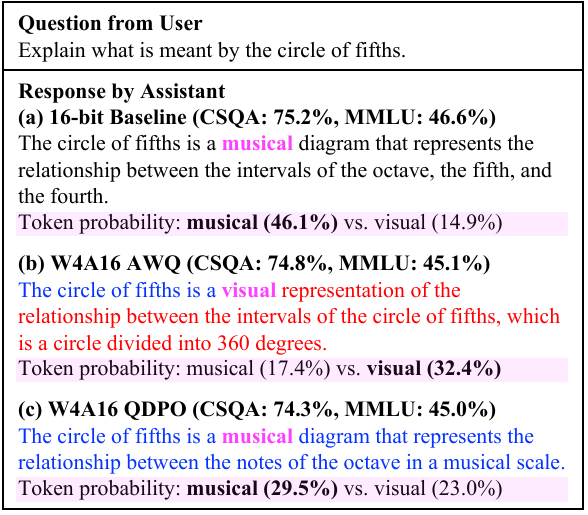}}
\caption{Example responses generated by Mi:dm-7B on 16-bit and 4-bit quantized inference.}
\label{fig:generation}
\end{figure}

Evaluating the conversational abilities of LLM-based chat assistants, especially for open-ended tasks requiring alignment with human preferences, challenges traditional score-based benchmarks due to the assistants' varied capabilities. To address this, new methods have been introduced for a more objective assessment of LLM chatbot performance ~\cite{vicuna2023,zheng2023judging}. The "LLM as a Judge" approach \cite{zheng2023judging} employs advanced LLMs like GPT-4~\cite{gpt4} to evaluate responsiveness in multi-turn conversations across eight conversational categories, focusing on conversational continuity and adherence to instructions. Furthermore, FLASK~\cite{ye2023flask} offers fine-grained evaluation criteria that dissect conversational skills linguistically. Yet, these methods mainly target full-precision chatbots, leaving the performance of cost-efficient quantized LLM chatbots less explored.

To assess quantization's effect on LLM-based chatbots' conversational abilities, we qualitatively compared the responses of quantized LLMs with a 16-bit baseline. Fig.~\ref{fig:generation} reveals that quantized models often fail to maintain engaging dialogues with repetitive phrases. We identify "token-flipping" — a phenomenon where quantization errors skew token distribution, causing incorrect token selection — as a crucial factor for this quality degradation. Traditional task-dependent evaluation metrics, such as Common Sense Question Answering (CSQA)~\cite{talmor2019commonsenseqa} and Massive Multitask Language Understanding (MMLU)~\cite{mmlu}, may not fully detect these nuances. For example, as shown in Fig.~\ref{fig:generation}(a) and (b), 16-bit and W4A16 inference exhibit similar task accuracy, but W4A16 inference produces responses that are not helpful to the user. This observation underscores the need for a new quantization approach that preserves user-perceived effectiveness beyond the task-dependent benchmarks.
%underscoring the need for more comprehensive evaluation methods to gauge user-perceived effectiveness.

% To investigate the impact of quantization on the conversational abilities of LLM-based chatbots, we compared the performance of Quantized LLMs using conventional task-dependent benchmarks and conversational benchmarks like MT-bench. As illustrated in Fig.~\ref{fig:generation}, we observed that while quantization techniques like AWQ may improve performance on task-dependent benchmarks, they significantly degrade conversational abilities. We propose the phenomenon of "token-flipping" as a factor contributing to this performance degradation. Token flipping occurs when quantization distorts the calculated token distribution during LLM inference, leading to a switch between the top-1 and top-2 tokens. This phenomenon can result in the generation of meaningless repetitions or seemingly plausible but incorrect information within the conversation. Such errors are not adequately captured by automatic metrics commonly used for LLM model evaluation, like Rouge or Perplexity, but they can significantly impact the perceived performance of LLMs when used by humans.

To address the issue of token-flipping in quantized LLMs, we propose a novel preference alignment method that aligns quantized LLMs with full-precision counterparts. Drawing inspiration from direct preference optimization (DPO) strategies~\cite{rafailov2023direct,liu2023statistical}, our approach generates preference datasets directly from the quantized LLM and its full-precision counterpart to implement quantization-aware optimization for preference-reflective weight adjustments. Our quantization-aware direct preference optimization (QDPO) method improves the disparity between the top-1 and top-2 logits of token distribution, reducing token-flipping, and fostering more relevant and consistent text output. We rigorously tested QDPO on two instruction-tuned LLMs, Vicuna~\cite{zheng2023judging} and Mi:dm~\cite{midm}, assessing their conversational performance in both English and Korean. The results, as illustrated in Fig.~\ref{fig:generation}(c), demonstrate that QDPO markedly enhances conversational abilities beyond those achieved with established quantization techniques.

% The contributions of this paper can be summarized as follows:
% \begin{itemize}
%     \item We evaluate the conversational abilities of quantized LLMs.
%     \item We investigate token-flipping, a root cause of degradation for conversational text generation.
%     \item We propose QDPO, a method to align the quantized LLM for recovering conversational abilities.
%     \item We conducted an extensive evaluation of QDPO on various quantized LLMs.
% \end{itemize}
\section{Background}
\label{sec:background}

\subsection{Conversational Ability of LLM}
% 소단락 1. LLM-based chatbots 및 converstaional ability / 2. Benchmark for conversational ability
In the pre-training phase, LLMs learn from a vast corpus of text data collected from various sources, including the internet, books, articles, and conversations~\cite{c4,bookcorpus,gao2020pile,refinedweb}. Through this process, they acquire extensive knowledge on a wide range of topics, which forms the foundation that enables LLMs to flexibly respond to diverse conversational subjects~\cite{opt,touvron2023llama,touvron2023llama2,gpt3}. Subsequently, LLMs develop the capability to follow instructions through instruction fine-tuning and learn to align with human preferences via RLHF~\cite{alpaca,longpre2023flan,chung2022scaling,mukherjee2023orca,ouyang2022training}. Through such processes, LLM-based chatbots like GPT-4~\cite{gpt4} and Vicuna~\cite{vicuna2023} have acquired the conversational ability to engage with humans on various topics over multiple turns, distinguishing them from conventional language models.

To evaluate LLM-based chatbots, it is essential to assess their conversational ability, which is their key capability. However, existing task-dependent benchmarks such as MMLU~\cite{mmlu} and HELM~\cite{liang2023holistic} do not adequately capture human preferences, rendering them insufficient for evaluating LLM-based chatbots. In response, proposals for new benchmarks such as MT-Bench~\cite{zheng2023judging} and FLASK~\cite{ye2023flask} are emerging, focusing on multi-turn questions or alignment with human preferences to effectively evaluate conversational abilities.

\subsection{LLM Quantization}
% PTQ: AWQ, GPTQ, QuIP, 
% QAT: TSLD, LLM-QAT, 
% Serving: TensorRT-LLM
% Gemini-nano?
% QLoRA, OWQ (parameter-efficient fine-tuning)
LLMs demand high serving costs due to their extensive number of parameters~\cite{gpt3}.
%Quantization reduces the memory size by representing the model's weights in lower bit-precision, thereby speeding up inference speed by reducing the memory load time~\cite{lin2023awq,frantar2023optq}. Thus, many studies focus on reducing weights in the LLMs to reduced-precision~\cite{lin2023awq,frantar2023optq,lee2023owq,tsld,enhancing}.
% Weight quantization techniques~\cite{lin2023awq,frantar2023optq,lee2023owq,tsld,enhancing} reduce the memory size by representing the model's weights in lower bit-precision, thereby lowering memory load time and speeding up inference. 
% Post-training quantization (PTQ) changes the model's weights directly to lower precision without additional training, offering cost benefits. However, due to concerns about accuracy loss, PTQ utilizes a portion of the training samples to calibrate and minimize the layer-wise quantization error through methods~\cite{shao2024omniquant,lin2023awq}.
Weight quantization techniques~\cite{lin2023awq,frantar2023optq,lee2023owq,tsld,enhancing} address this issue by representing the model's weights in lower bit-precision, thereby reducing memory size, lowering memory load time, and speeding up inference. Post-training quantization (PTQ) changes the model's weights directly to lower precision without additional training, offering cost benefits. However, due to concerns about accuracy loss, PTQ utilizes a portion of the training samples to calibrate and minimize the layer-wise quantization error through methods such as AWQ~\cite{lin2023awq}. Quantization-aware training (QAT), on the other hand, maintains the performance of a quantized model by applying quantization during the forward pass and training the model accordingly. When applying QAT to LLMs, due to the insufficient information from the ground truth, techniques often use Knowledge Distillation (KD) by reducing the distance between the logits of the quantized model and the full-precision model~\cite{tsld,liu2023llmqat}. 

% Quantization-aware training (QAT) maintains the performance of a quantized model by training the model with quantization applied during the forward pass. When applying QAT to LLMs, due to the insufficient information from the ground truth, techniques often use a Knowledge Distillation (KD) by reducing the distance between the logits of the quantized model and the full-precision model~\cite{tsld,liu2023llmqat}. 

% However, these quantization studies have limitedly evaluated the conversational abilities of quantized LLMs or have not explained how the quantization error affects the model's conversational ability. For example, AWQ~\cite{lin2023awq}, in their paper, reports a loss rate of 60\% for the W3A16 quantized model compared to the baseline evaluated from GPT-4~\cite{gpt4}, despite using fine-grained group quantization (with a group size of 128 for a hidden dimension of 4,096 in Vicuna-7B~\cite{vicuna2023}).
% In this research, we analyze how the model's quantization noise impacts the LLM's conversational abilities and propose methods to enhance the model's conversational capability.

However, previous quantization studies have evaluated their methods on task-dependent benchmarks, which show a limited scope for comprehensive evaluation of conversational abilities. For example, AWQ~\cite{lin2023awq} emphasizes that the quantized model achieves accuracy comparable to the baseline on CSQA. However, they do not analyze why only 35\% of the sentences generated by the quantized model are considered as good as those from the baseline, according to GPT-4~\cite{gpt4}’s evaluation in assessing conversational abilities. In this research, we analyze how the model’s quantization error impacts the conversational abilities of large language models and propose methods to enhance these capabilities.

\subsection{Alignment with Human Preferences}
%RLHF~\cite{ouyang2022training} aiming to align complex AI systems more closely with human preferences. Its main advantage lies in leveraging human judgment for evaluating appropriate behavior rather than relying on demonstrations or manually defined rewards. This method has gained significant traction, especially in fine-tuning LLMs, despite facing challenges such as data quality issues, the risk of reward misgeneralization, reward hacking, and complexities in policy optimization. 
The RLHF is an advanced method to improve the performance of LLMs by aligning with human preferences. It comprises three stages:

\textbf{Supervised Fine-Tuning (SFT).} SFT utilizes a dataset of human instructions to refine pre-trained LLMs.

\textbf{Reward Modeling.} This stage develops a reward model based on human preferences for LLM response pairs, using the Bradley-Terry (BT) model \cite{bradley1952rank} to quantify these preferences. It represents the distribution of preferences distribution $p^*$ between $y_1$ and $y_2$:
\begin{equation}
\label{eq:BT}
% p^*(y_1\succ y_2 | x) = 
% \frac{\exp(r^*(x, y_1))}{\exp(r^*(x, y_1)) + \exp(r^*(x, y_2))},
p^*(y_1 \succ y_2 | x) = \frac{e^{r^*(x,y_1)}}{e^{r^*(x,y_1)} + e^{r^*(x,y_2)}}, 
\end{equation}

\(r^*\) is defined as the optimal reward function. \(y_1\) and \(y_2\), assumed to be sampled from the optimal preference distribution \(p^*\) with prompt \(x\), the parameterized reward model estimates the parameter using maximum likelihood.\\

\textbf{Policy Optimization.} The LLM policy optimization is guided by the reward model to generate responses that better align with human preferences for the training prompts.
%통상적인 objective fuction은 다음과 같이 정의된다:
The reinforcement learning (RL) objective function is defined as follows:
\begin{equation}
\label{eq:rlhf}
\max_{\pi} \underset{\substack{x \sim D\\ y \sim \pi}}{\mathbb{E}} \left[ r(x, y) \right] - \beta D_{KL} \left[ \pi(y|x) \| \pi_{\text{ref}}(y|x)\right],
\end{equation}
\(\pi\) represents the LLM policy, \(\beta\) is a control parameter that regulates variations with respect to \(\pi_\text{ref}\). 
%Essentially, within the context of the provided data \((x, y)\), the objective function aims to maximize the reward while ensuring minimal alterations from \(\pi_\text{ref}\). 
Recent approach~\cite{ouyang2022training} employs Proximal Policy Optimization ~\cite{schulman2017proximal} for RL-based optimization, wherein the necessary reward is derived from a previously trained reward model.
%\pi는 LLM policy 이며,  \beta는 pi_ref와의 변화를 조절하는 parameter이다. 즉, 주어진 data (x,y) 에서 \pi_{ref}와 크게 변화가 없는 제약하에 reward를 최대화 하는 objective function이며, 최적화를 위해 최근 접근들은 PPO와 같은기법들을 사용하며, 이 때 필요한 reward는 앞서 학습된 reward모델로 부터 얻는다.
% \subsection{DPO}

%RLHF와 같은 prefernce와 llm의 alignmnet를 목적으로 가지면서 RL과는 다른 접근인 direct preference optimzation 은 preference를 supervised-learning으로 직접 policy에 학습하는 방법을 제안한다. eq.~({eq:rlhf})으로 부터 다음과 같은 formulation을 유도한다: 
% DPO \cite{rafailov2023direct} offers an alternative approach to aligning preferences with LLM policies by directly incorporating preferences into the policy through supervised learning. %The DPO loss originates from a formulation derived from Eq.~(\ref{eq:rlhf}):
% They starting from eq.~(\ref{eq:rlhf}), demonstrates the ability to express the optimal reward in terms of concerning the optimal policy:
DPO \cite{rafailov2023direct} aligns LLM policies with human preferences via supervised learning, leveraging Eq.~(\ref{eq:rlhf}) to relate the optimal reward to the optimal policy directly.
\begin{multline}
\label{eq:opt_reward}
r^*(x, y) = \beta \log \left( \frac{\pi^*(y|x)}{\pi_{\text{ref}}(y|x)} \right) + \beta \log Z(x),
\end{multline}
where \(Z(x)\) is the partition function.
% % The DPO originating from eq.~(\ref{eq:rlhf}), begins with the formulation of the optimal policy:
% % \begin{multline}
% % \label{eq:opt_policy}
% % \pi^*(y|x) = \frac{1}{Z(x)} \pi_{\text{ref}}(y|x) \exp\left(\frac{r(x, y)}{\beta}\right),
% % \end{multline}
% \(Z(x)\) is the partition function, and by reorganizing the eq.~(\ref{eq:opt_policy}), the optimal reward function is defined from the optimal policy $\pi^*$ as follows:
The optimal reward function is fitted to the objective function of BT model, defining DPO loss as follows:
% \begin{multline}
% \label{eq:DPO_BT}
% %p^*(y_1\succ y_2 | x) =
%  \frac{1}{1 + \exp\left( \beta \log \frac{\pi^*(y_2|x)}{\pi_{\text{ref}}(y_2|x)} - \beta \log \frac{\pi^*(y_1|x)}{\pi_{\text{ref}}(y_1|x)} \right)}
% \end{multline}
\begin{multline}
\label{eq:DPO}
\underset{{x, y_w, y_l \sim D}}{\mathbb{E}} \left[ - \log \sigma \left( \beta \log \frac{\pi_{\theta}(y_w | x)}{\pi_\text{ref}(y_w | x)} \right.\right.\\\left.\left.- \beta \log \frac{\pi_{\theta}(y_l | x)}{\pi_\text{ref}(y_l | x)} \right)\right],
\end{multline}
where $\sigma$ is logistic function.

SRO~\cite{liu2023statistical} criticizes the preference sampling method of DPO. Sampling data $y_w$ and $y_l$ from the $\pi^*$ is the optimal way for estimating $\pi_\theta$. However, all experiments in DPO use preference pairs not from the $\pi^*$ but from $\pi_\text{ref}$, and there is a lack of research into the implications of this approach. SRO proposes a solution by constructing an additional reward-ranking model to directly form preference pairs from an approximated optimal policy and statistical rejection sampling. 

\begin{figure*}[t]
% \centerline{\includegraphics[width=\textwidth]{Figures/breakdown.pdf}}
\centerline{\includegraphics[width=\textwidth]{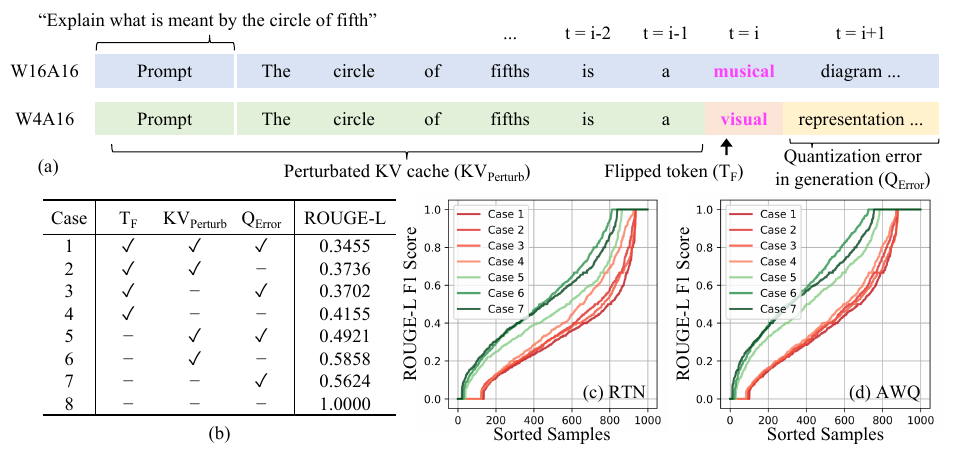}}
\vspace{-4mm}
\caption{(a) Breakdown of factors influencing sentence generation in quantized models. (b) Case study on the impact of each factor. The ROUGE-L score is used to measure changes in sentences. More results for ROUGE-1/2 are in Fig.~\ref{fig:rouge-1-2}. (c-d) Case-wise ROUGE scores in models where W4A16 PTQ is applied with (c) RTN and (d) AWQ.}
\label{fig:breakdown}
\end{figure*}

\section{Conversational Abilities of Quantized LLMs}
% \section{Impact of Quantization on Generation}
\label{sec:analysis}
\subsection{Observations}
Recent advancements in PTQ have demonstrated that 4-bit quantized LLMs are effective for a variety of tasks, as evidenced by references such as AWQ and OPTQ~\cite{lin2023awq,frantar2023optq}. However, our observations reveal that these quantized LLMs struggle to sustain engaging conversations, particularly in multi-turn chatbot interactions. For instance, Fig.~\ref{fig:generation} illustrates the contrast between the 16-bit baseline and 4-bit quantized LLMs in sentence generation. The baseline model begins its responses with ``The circle of fifths is a musical diagram,'' providing relevant answers. On the other hand, the 4-bit quantized model starts to deviate at the seventh token, switching its focus from ``musical'' to ``visual,'' and often generates limited and repetitive phrases. Although both models display similar task performance metrics, such as accuracy in multiple-choice benchmarks, there's a noticeable difference in the logit probability for the seventh token in the 4-bit model, causing a change in the token from ``musical'' to ``visual.'' This issue of altered text generation, observed across multiple examples (see \ref{appendix:examples} for additional examples), prompts an investigation 
into its underlying causes.

\subsection{Breakdown Analysis}
\label{sec:conversation:breakdown}
To understand the cause of altered text generation in quantized LLMs, we examine how quantization impacts text production. We pinpoint the initial deviation to a flipped token and identify three contributing factors as shown in Fig.~\ref{fig:breakdown}(a):
\begin{itemize}
  \item[-] \textbf{Flipped token ($\mathrm{T_F}$):} Occurs when a quantized model selects a different token at timestep  $t=i$ compared to the baseline, altering the input for subsequent token generation and leading to deviations. 
  \item[-] \textbf{Perturbated KV cache ($\mathrm{KV_{Perturb}}$):} Despite identical token sequences up to timestep $t=i-1$, quantization errors already affect the Transformer's key-value caches, contributing to further deviations. 
  \item[-] \textbf{Quantization error in generation ($\mathrm{Q_{Error}}$):} Starting from timestep $t=i+1$, ongoing quantization errors continue to influence token generation, causing further divergence from the baseline.
\end{itemize}
% Based on previous observations, we analyze the impact of quantization on the model and its effect on different generations by categorizing into three distinct factors, as depicted in Fig.~\ref{fig:breakdown}(a). The details for each factor are as follows:
% \begin{itemize}
%   \item[-] \textbf{Flipped token:} The Quantized model generates the same tokens as the baseline model until a specific timestep $t$, where a quantization error leads to a modified probability, resulting in the generation of a different token. We define this initially changed token as the flipped token $\mathrm{T_F}$.
%   \item[-] \textbf{Perturbated KV cache:} Up to timestep $t-1$, the quantized model generates the same as the baseline model, but the key and value of past tokens affecting attention operation exist in a state influenced by quantization; we define this as the perturbated KV cache ($\mathrm{KV_{Perturb}}$).
%   \item[-] \textbf{Quantization error in generation:} From timestep $t+1$ onwards, the model performs generation with different inputs from the baseline model, during which the impact of quantization error applied to the model weight exists. We define this factor as $\mathrm{Q_{Error}}$.
% \end{itemize}

\textbf{Setup.} 
To evaluate the impact of each identified factor, we analyze eight possible scenarios shown in Fig.~\ref{fig:breakdown}(b). Case 1, where all three factors are present, mirrors the standard text generation of a quantized LLM, whereas Case 8, devoid of these factors, corresponds to the baseline model's inference. For this analysis, we generate text using both 4-bit and 16-bit models with 1,000 instruction samples randomly chosen from the Alpaca dataset~\cite{alpaca}. We record the first token where discrepancies in text generation between the two models occurred, along with the key-value cache status up to that point for each scenario. To quantify the deviation from the baseline text, we utilize the ROUGE-L~\cite{lin-2004-rouge} as a metric (the higher the better). More details of the implementation for the breakdown analysis are provided in \ref{appendix:breakdown}.
%To ascertain the effects of each factor, we conduct a case study on the possible eight scenarios as depicted in Fig.~\ref{fig:breakdown}(b). The scenario where all three factors are in effect (Case 1) corresponds to quantized inference, while the scenario with none of the factors (Case 8) aligns with baseline inference. We utilize 1K randomly selected instruction samples from the Alpaca dataset~\cite{alpaca} as prompts to generate outputs using both 4-bit and 16-bit models. We dump the first token where the two models diverge in generation and the Key-Value cache leading up to that point for each case study. To evaluate how much the overall sentence changes from the baseline inference, we employ ROUGE-L~\cite{lin-2004-rouge} as our evaluation metric.

\begin{figure}[t]
\centerline{\includegraphics[width=\columnwidth]{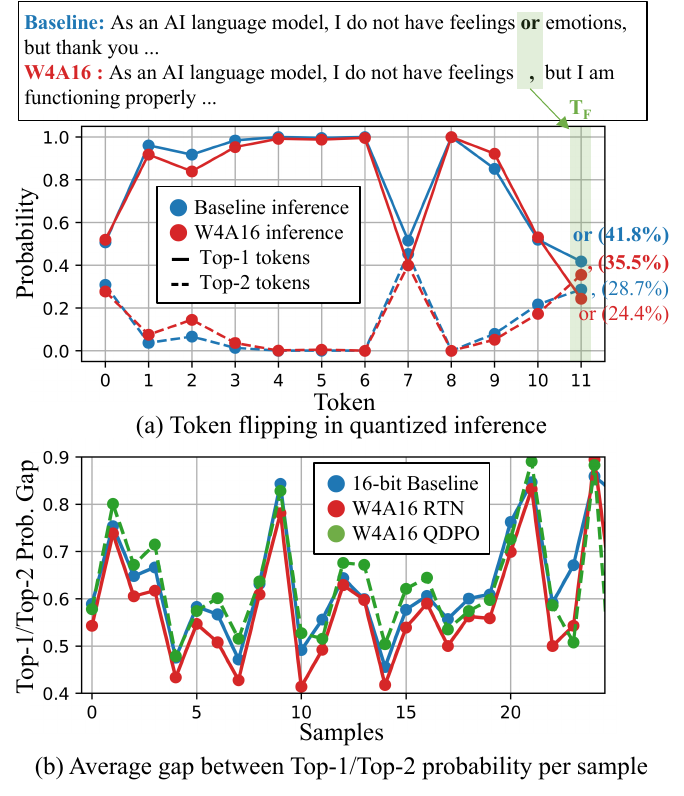}}
\caption{(a) Auto-regressive inference probabilities for baseline and quantized models, token by token. (b) Difference in average probability between top-1 and top-2 tokens per sample (Mi:dm, from MT-Bench). See Fig.~\ref{fig:prob_gap_awq} for more on the AWQ case.}
\label{fig:flip_token}
\vspace{-2.5mm}
\end{figure}

\textbf{Results.} To assess the contribution of each factor to deviations in sentence generation from the baseline, we contrast each scenario with Case 8. As depicted in Fig.~\ref{fig:breakdown}(b), sentences become more divergent with the inclusion of additional factors. Specifically, from Case 4, it is evident that the Flipped Token ($\mathrm{T_F}$) significantly affects sentence variation, as indicated by the largest decrease in ROUGE scores. Conversely, the effects of perturbed KV cache ($\mathrm{KV_{Perturb}}$) and quantization error in generation ($\mathrm{Q_{Error}}$) are comparatively minor. This pattern is further highlighted in Fig.~\ref{fig:breakdown}(c), where ROUGE scores, sorted by samples, show that Cases 1-4 cluster on the right, signifying greater deviation. This suggests that even a single token difference, resulting from quantization-affected probability shifts, can substantially alter the overall sentence structure in quantized inference.
%To measure how much each factor diverges in sentence generation from the baseline inference, we compare each case against Case 8. As illustrated in Fig.~\ref{fig:breakdown}(b), the involvement of more factors results in the generation of more divergent sentences. Fig.~\ref{fig:breakdown}(c) displays the comparison of ROUGE scores, sorted across all samples. A very clear trend is observed: Cases 1 to 4, which are influenced by $\mathrm{T_F}$, generate the most divergent sentences. In contrast, the effects of $\mathrm{KV_{Perturb}}$ or $\mathrm{Q_{Error}}$ alone are relatively minor. Notably, Case 4, where only the effect of $\mathrm{T_F}$ is present, results in less sentence variation than Case 5, where both $\mathrm{KV_{Perturb}}$ and $\mathrm{Q_{Error}}$ are in effect. This demonstrates that in quantized inference, a single token generated differently due to probability changes can significantly impact the overall sentence variation.

\textbf{Ablation: Advanced Quantization. } Advanced quantization techniques designed to minimize errors may not fully address the issue of deviated text generation caused by flipped tokens. Recent PTQ methods that employ calibration using a small sample by scaling weight channels or adjusting quantization step sizes~\cite{frantar2023optq,lin2023awq,enhancing} aim to lessen layer-specific quantization errors. However, our observations indicate that while these calibrated PTQ models reduce quantization error effects, they do not mitigate the issues stemming from flipped tokens. The case study for a 4-bit quantized model calibrated with AWQ~\cite{lin2023awq}, shown in Fig.~\ref{fig:breakdown}(d), reveals that although calibration decreases the impacts of $\mathrm{KV_{Perturb}}$ and $\mathrm{Q_{Error}}$, sentence variations are still predominantly influenced by $\mathrm{T_F}$. A similar trend can be observed by KD-based QAT (Fig.~\ref{fig:rouge-kd}), highlighting the need for strategies that specifically address flipped tokens.
% PTQ directly represents the model's weights in low-bit format, leading to significant performance degradation. To mitigate this, PTQ research utilizes a small sample for calibration, scaling each channel of the weights or adjusting the quantization step size~\cite{gptq,awq,owq,enhancing}. The primary objective of these efforts is to reduce the quantization error occurring at each layer. Intriguingly, we observe that while calibrated PTQ models effectively reduce the impact of quantization error, they do not improve the effects caused by the inevitable flipped tokens. Fig.~\ref{fig:breakdown}(d) presents the case study results for a 4-bit model with calibration~\cite{awq} applied. Compared to Fig.~\ref{fig:breakdown}(c), it is evident that calibration reduces the influence of $\mathrm{KV_{Perturb}}$ and $\mathrm{Q_{Error}}$ by diminishing quantization error. However, more notably, the overall difference in sentences is significantly impacted by $\mathrm{T_F}$, suggesting the necessity of effectively controlling the flipped token.

\subsection{Why Token-Flipping Happens?}
We hypothesize that token-flipping occurs due to inherently ambiguous token distributions in sentence generation, which become prone to flipping when quantization errors introduce alterations. To empirically validate this, Fig.~\ref{fig:flip_token}(a) demonstrates token-flipping during text generation by a quantized model. It shows the probabilities for the top-1 and top-2 tokens throughout the auto-regressive generation. Notably, the 16-bit baseline and 4-bit quantized models produce nearly identical probabilities for most tokens. However, at certain points (e.g., $t=0, 7, 11$), the probability margin between the top-1 and top-2 tokens is minimal. Token-flipping occurs when quantization-induced deviations in the probability distribution surpass this narrow margin, altering subsequent sentence generation and leading to unnatural phrasing.

Fig.~\ref{fig:flip_token}(b) shows the average probability margin between the top-1 and top-2 tokens across each text sample. By feeding identical inputs to each model, we note that the 4-bit quantized model has a narrower average probability margin between the top-1 and top-2 tokens than the 16-bit baseline. This indicates a higher likelihood of the 4-bit model experiencing token-flipping due to quantization error-induced deviations exceeding this margin. Additionally, our examination of beam search~\cite{graves2012sequence} in Section~\ref{sec:ablation} reveals its limited effectiveness in mitigating this issue. This underscores the need for strategies that ensure the quantized model retains clear decision-making capabilities.

\section{QDPO: Quantization-aware Direct Preference Optimization}
\label{sec:qdpo}

%\subsection{Align Quantized LLM with Full-precision LLM}

% As described in section~\ref{sec:analysis}, quantization leads to a significant degradation in the conversation ability of LLMs. To address this issue, we introduce an algorithm named Quantization-aware Direct Preference Optimization (QDPO). This algorithm aims to align the conversation ability of quantized LLMs with that of LLMs before quantization.

As described in Section~\ref{sec:analysis}, quantization significantly degrades the conversational ability of LLMs. To address this issue, we introduce an algorithm named Quantization-aware Direct Preference Optimization (QDPO), which aims to align the conversational abilities of quantized LLMs with those of LLMs prior to quantization. 
% \blue{QDPO introduces an efficient method for generating the dataset \(\mathcal{D}_\text{{QDPO}}\) without costly human annotations and establishes a theoretical foundation that enables the automatic distinction of preferences during dataset generation.}
% \blue{QDPO has two main contributions: 1) It provides an efficient method for generating the dataset \(\mathcal{D}_\text{{QDPO}}\) without costly human annotations. 2)  It offers a theoretical foundation that ensures the automatic distinction of preferences during dataset generation.}
QDPO has two main contributions: 1) Providing an efficient method for generating the dataset \(\mathcal{D}_\text{{QDPO}}\) without costly human annotations. 2) Offering a theoretical foundation that ensures the automatic distinction of preferences during dataset generation.

\subsection{Method}

Drawing inspiration from the success of DPO in aligning LLMs with human preferences, we have developed a novel approach that extends its application to overcome the challenges introduced by quantization.

The challenge in preference dataset generation arises from human labeling. 
To mitiagte this, we introduce for efficiently creating dataset \(\mathcal{D}_{\text{QDPO}}\), which is composed of triplets \(\{y_w, y_l, x\}\). Here, \(y_w\) denotes the response from the full-precision model \(\pi_\text{fp}\), which is also referred to as the optimal policy. \(y_l\) represents the corresponding response from the quantized model \(\pi_{\text{q}}\). \(x\) serves as the prompt. Specifically, \(y_w\) is obtained as \(\arg\max_y{\pi_\text{fp}(y|x)}\) and \(y_l\) as \(\arg\max_y{\pi_\text{q}(y|x)}\).
The preference of \(y_w\) over \(y_l\) is ensured by Theorem~\ref{thm:D_QDPO}. The proof is in \ref{proof2}. Unlike conventional DPO methods, QDPO automatically distinguishes preferences without relying on expensive human-annotated datasets.

%The preference hierarchy between \(y_w\) and \(y_l\) is guaranteed under the following theorem~\ref{thm:D_QDPO}.
% \begin{theorem}
% %If \pi_\text{fp} represents the preference p^*, then for y_1 = \arg\max_y{\pi_\text{fp}(y|x)} and y_2 =\arg\max_y{\pi(y|x)}, it holds that p^*(y_1 \succ y_2) \geq p^*(y_2 \succ y_1).
% \end{theorem}
% \begin{proof}
% $\forall y \in Y, \pi_\text{fp}(a_1|s) \ge \pi_\text{fp}(a|s)$ 
% \end{proof}
\begin{theorem}
    \label{thm:D_QDPO}
     For any response \(y\) in the set of all possible responses \(Y\), if \(y_1 = \arg\max_{y \in Y} \pi_\text{fp}(y|x)\) and \(y_2 = \arg\max_{y \in Y} \pi_\text{q}(y|x)\), then it is guaranteed that \(p^*(y_1 \succ y_2) \geq p^*(y_2 \succ y_1)\).
\end{theorem}
%We include the proof in \ref{proof2}

By precisely distinguishing preferences, we can clearly eliminate errors in data labeling. This leads to improved performance by accurately estimating the policy model's density. \(\mathcal{L}_{\text{QDPO}}\) is define with high-quality preference data \(\mathcal{D}_{\text{QDPO}}\) as follows:

\begin{multline}
\label{eq:QDPO}
\underset{\substack{x \sim \mathcal{D}_\text{{QDPO}} \\ y_{w}\sim \pi_\text{fp},y_{l}\sim \pi_\text{q}}}{\mathbb{E}}\left[ - \log \sigma \left( \beta \log \frac{\pi_{\theta}(y_w | x)}{\pi_{\text{q}}(y_w | x)}\right.\right. \\\left.\left.- \beta \log \frac{\pi_{\theta}(y_l | x)}{\pi_{\text{q}}(y_l | x)} \right) \right].
\end{multline}

\begin{figure}[t]
% \centerline{\includegraphics[width=0.9\columnwidth]{Figures/fig_loss3.png}}
\centerline{\includegraphics[width=\columnwidth]{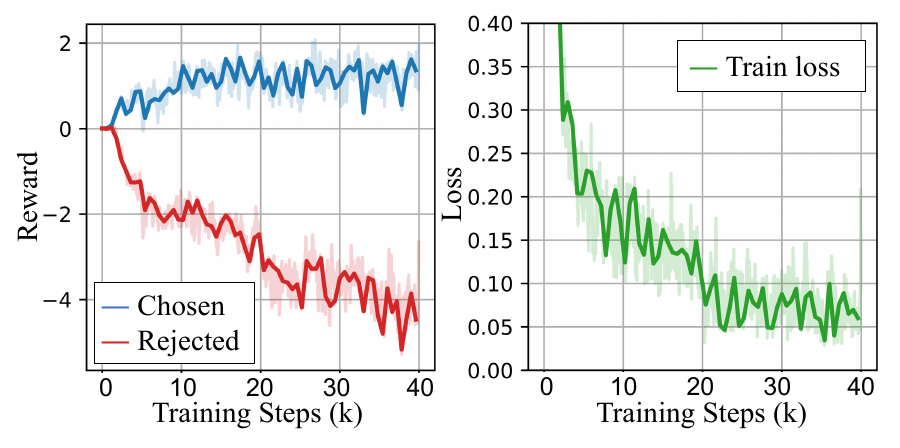}}
\vspace{-3mm}
\caption{Training dynamics of QDPO showing chosen and rejected rewards (left), and loss (right) across steps.}
\label{fig:qdpo_loss}
\end{figure}

\begin{algorithm}[t]
\label{algo:alg1}
\caption{Quantization-aware DPO}
\begin{algorithmic} % The number [1] indicates that lines are to be numbered
\State \textbf{Input:} prompt $\{x_1, x_2, \dots, x_N\}$, full precision policy $\pi_\text{fp}$, quantized policy $\pi_\text{q}$, KL penalty $\beta$
\State \textbf{Output:} Updated policy $\pi_\theta$

\State Initialize $\pi_\theta$ from $\pi_\text{q}$
\State Preference pairs dataset $\mathcal{D}_\text{{QDPO}} = \emptyset$

\For{$i = 1$ to $N$}
    \State $y^i_w =\arg\max_y{\pi_\text{fp}(y|x^i)}$
    \State $y^i_l = \arg\max_y{\pi_\text{q}(y|x^i)}$
    \State Add pair $\{y^i_w, y^i_l, x^i\}$ to $\mathcal{D}_\text{{QDPO}}$
\EndFor

\For{each pair $\{y_w, y_l, x\}$ in $\mathcal{D}_\text{{QDPO}}$}
    \State Calculate $\mathcal{L}_\text{QDPO}$ from eq.~(\ref{eq:QDPO})    
    \State Calculate the gradient with respect to $\theta$
    \State $\frac{\partial \mathcal{L}_\text{QDPO}}{\partial \theta} = \frac{\partial \mathcal{L}_\text{QDPO}}{\partial \theta_\text{q}} \cdot \frac{\partial \theta_\text{q}}{\partial \theta}  \underset{\text{{STE}}}{\approx} \frac{\partial \mathcal{L}_\text{QDPO}}{\partial \theta_\text{q}} $
    \State Update $\pi_\theta$ by minimizing $\mathcal{L}_\text{QDPO}$
\EndFor

\State \textbf{return} Updated policy $\pi_\theta$
\end{algorithmic}
\end{algorithm}
\subsection{Implementation}

Given $\pi_{\theta}$ as the quantized model's policy, integrating $\mathcal{L}_\text{QDPO}$ with QAT adjusts for quantization effects. The quantization technique we employ uniformly quantizes each channel across its entire min-max range, ensuring comprehensive accommodation of the full spectrum of values within each channel. To overcome the challenge posed by the non-differentiable rounding within the quantization process, we employ the Straight-Through Estimator (STE) for gradient approximation, facilitating effective gradient approximation and ensuring smooth training despite quantization. As shown in Fig.~\ref{fig:qdpo_loss}, QDPO demonstrates an increase in the chosen reward and a decrease in the rejected reward throughout the training process, indicating effective loss convergence. The complete procedure of QDPO is described in Algorithm 1. Details of the training settings and hyperparameters for QDPO can be found in \ref{appendix:training_details}.

\section{Experiments}
\label{sec:experiments}
\subsection{Experimental Settings}

% \textbf{Models} To evaluate the effectiveness of the proposed method, we evaluate various instruction fine-tuned LLMs with 7B parameters. We assess representative LLMs, LLaMA2-Chat~\cite{touvron2023llama2} and Vicuna~\cite{vicuna2023}, which apply instruction fine-tuning to the LLaMA2~\cite{touvron2023llama2} model. Specifically, we focus on evaluating the bilingual LLM, Mi:dm\footnote{\url{https://huggingface.co/KT-AI/midm-bitext-S-7B-inst-v1}}, to confirm QDPO's effectiveness in models that support multiple languages.\\
% \textbf{Models} We evaluate QDPO on three representative conversational LLMs. Vicuna-7B-v1.5~\cite{zheng2023judging} and LLaMA2-Chat-7B~\cite{touvron2023llama2} are instruction-finetuned from LLaMA2 for improved conversational ability. We also include a bilingual (English-Korean) LLM, Mi:dm-7B\footnote{\url{https://huggingface.co/KT-AI/midm-bitext-S-7B-inst-v1}}, to confirm QDPO's effectiveness to support multiple languages. All these models have 7B parameters.

\textbf{\quad Models.} We evaluate QDPO on two representative conversational LLMs. Vicuna-v1.5~\cite{zheng2023judging}, instruction-finetuned from LLaMA2 for improved conversational ability, and a bilingual (English-Korean) LLM, Mi:dm~\cite{midm}, to confirm QDPO's effectiveness to support multiple languages. All these models have 7B parameters.

% \textbf{Benchmarks} To evaluate the conversational ability of quantized LLMs, we use benchmarks that high-performance LLMs, such as GPT-4~\cite{gpt4}, employ to assess the quality of generated sentences. We evaluate our method across three benchmarks: MT-Bench~\cite{zheng2023judging}, Vicuna-Eval~\cite{vicuna2023}, and FLASK~\cite{ye2023flask}. MT-Bench utilizes GPT-4 to evaluate the quality of two responses obtained from an initial question and an additional follow-up question, offering a more comprehensive evaluation of a language model's conversational abilities compared to single-turn questions. For assessing Korean capability, we translate the MT-Bench dataset into Korean using GPT-4. Vicuna-Eval determines which model generates better sentences by having GPT-4 evaluate the answers from two models, and it consists of 80 questions. FLASK includes a total of 1.7K samples, each designed to comprehensively assess various skills of the model, such as robustness and harmlessness.\\
\textbf{Benchmarks.} For a comprehensive evaluation of conversational abilities, we employ three distinct benchmarks: MT-Bench~\cite{zheng2023judging}, Vicuna-Eval~\cite{vicuna2023}, and FLASK~\cite{ye2023flask}. MT-Bench utilizes GPT-4 to evaluate the quality of two responses obtained from an initial question and an additional follow-up question, offering an evaluation of multi-turn responses. For assessing Korean capability, we also translate the MT-Bench dataset into Korean using GPT-4. Vicuna-Eval consists of 80 questions for evaluation by GPT-4 to determine which model generates better sentences. FLASK includes 1.7K samples designed to assess LLM's fine-grained language skills, such as robustness and harmlessness.

\textbf{Quantization Methods.} To understand the impact of quantization on conversational abilities, we consider variations of quantization methods:
\begin{itemize}[itemsep=-1pt, topsep=0pt]
    \item[-] Baseline: 16-bit floating-point weight
    \item[-] RTN~\cite{jacob2018quantization}: 4-bit round-to-nearest weight quantization
    \item[-] AWQ~\cite{lin2023awq}: 4-bit RTN with weight scaling for improved quantization
    \item[-] KD~\cite{liu2023llmqat}: 4-bit quantization-aware training with knowledge distillation (KD) loss from Baseline
    \item[-] QDPO (Ours): 4-bit RTN with QDPO for improved conversational abilities
\end{itemize}
Details of the experimental settings for each case can be found in \ref{appendix:training_details}.

% \usepackage{multirow}
% \usepackage{graphicx}
\begin{table}[t]
\centering
\resizebox{0.9\columnwidth}{!}{%
\begin{tabular}{c|c|c|ccc|c}
\Xhline{3\arrayrulewidth}
Lang.                 & Model                                                                   & Method & Win                 & Tie                 & Lose                 & Lose-rate ↓           \\ \midrule
\multirow{7}{*}{Eng} & \multirow{4}{*}{Mi:dm}                                                   & RTN    & 24                   & 6                    & 66                    & 0.69                 \\
                      &                                                                         & AWQ    & 28                   & 9                    & 52                    & 0.58                 \\
                      &                                                                         & KD     & 31                   & 16                   & 52                    & 0.53                 \\
                      &                                                                         & QDPO   & 53                   & 14                   & 44                    & \textbf{0.40}        \\ \cline{2-7} 
                      & \multirow{3}{*}{Vicuna}                                                 & RTN    & 26                   & 22                   & 73                    & 0.60                 \\
                      &                                                                         & AWQ    & 39                   & 22                   & 47                    & \textbf{0.44}        \\
                      &                                                                         & QDPO   & 40                   & 27                   & 53                    & \textbf{0.44}        \\ \midrule
                    
\multirow{3}{*}{Kor}  & \multirow{3}{*}{Mi:dm}                                                   & RTN    & 29                   & 7           & 55                    & 0.60                 \\
                      &                                                                         & AWQ    & 25                   & 5                    & 48                    & 0.62                 \\
                      &                                                                         & QDPO   & 45                   & 4                    & 61                    & \textbf{0.55}        \\ \Xhline{3\arrayrulewidth}
\end{tabular}%
}
\caption{Pairwise comparison results in MT-Bench between W4A16 quantized LLMs and 16-bit baseline model.}
\label{tab:mtbench_pair}
% \vspace{-3mm}
\end{table}

\begin{table}[t]
\centering
\resizebox{0.9\columnwidth}{!}{%
\begin{tabular}{c|c|ccc}
\Xhline{3\arrayrulewidth}
\multirow{2}{*}{Category} & \multirow{2}{*}{16-bit Inference} & \multicolumn{3}{c}{W4A16 Inference} \\
                          &                                   & RTN     & AWQ     & QDPO            \\ \midrule
Writing                   & 5.82                              & 4.13    & 5.39    & 4.74            \\
Roleplay                  & 5.61                              & 5.53    & 5.00    & 5.13            \\
Reasoning                 & 3.37                              & 3.06    & 3.61    & 4.31            \\
Math                      & 1.71                              & 1.45    & 1.60    & 1.40            \\
Coding                    & 1.11                              & 1.56    & 1.16    & 2.28            \\
Extraction                & 3.63                              & 2.56    & 3.50    & 3.08            \\
STEM                      & 5.24                              & 4.39    & 4.68    & 5.69            \\
Humanities                & 6.26                              & 5.75    & 5.00    & 5.63            \\ \midrule
Average                   & 4.09                              & 3.55    & 3.74    & \textbf{4.03}   \\ \Xhline{3\arrayrulewidth}
\end{tabular}%
% \vspace{-1mm}
}
% \vspace{-2mm}
\caption{Category-wise scores of quantized LLMs according to MT-Bench single-answer grading.}
% \vspace{-3mm}
\label{tab:mtbench_single_midm}
\end{table}
\subsection{Experimental Results: MT-Bench}
We evaluate quantized LLMs on MT-Bench to understand the impact of different quantization methods on conversational abilities. Following convention~\cite{zheng2023judging}, we report both pairwise comparison and single-answer grading results (\ref{appendix:mt-bench-score} for detailed evaluation metrics).

\textbf{Pairwise Comparison. } 
%As shown in Table~\ref{tab:mtbench_pair}, AWQ generally reduces the lose-rate compared to RTN by minimizing the quantization error through calibration. However, there are still more losses than wins. KD shows a lower lose-rate by further training to follow the baseline model, but QDPO exhibits the lowest lose-rate by enhancing the conversational ability of the quantized model.
Table~\ref{tab:mtbench_pair} shows the results of pairwise comparison on MT-Bench for various quantized LLMs. Each quantized LLM is compared with the Baseline (16-bit weight) by GPT-4 for their multi-turn responses to the questions in various categories of MT-Bench. We focus on the lose-rate since our alignment objective is to improve the answer quality of the quantized LLM superior to (win) or comparable with (tie) the 16-bit weight baseline. In all the cases, RTN suffers from the highest lose-rate compared to AWQ due to its simplest quantization mechanism. However, lose-rate of the same RTN can be significantly improved by QDPO; QDPO achieves the lowest lose-rate in all the cases. We can further compare QDPO with KD as they finetune the model weights to be quantization-friendly. Interestingly, QDPO outperforms KD with a noticeable increase in winning cases. These results showcase that QDPO can effectively align the answer quality to the 16-bit weight baseline.

\textbf{Single-Answer Grading.} Table \ref{tab:mtbench_single_midm} presents the single-answer grading results of Mi:dm on MT-Bench across eight categories, each with 10 questions, and reports the average GPT-4 rating (higher is better). Throughout the categories, RTN suffers from the lowest grading due to the quantization errors, which can be marginally improved by AWQ. In contrast, QDPO significantly improves the average grading from RTN, achieving the average grading on par with the 16-bit weight baseline. This also highlights the effectiveness of QDPO in recovering conversational abilities. Details on category-wise analysis can be found in \ref{appnedix:mt-bench-category-analysis}. 

\subsection{Experimental Results: Vicuna-Eval}

\begin{figure}[t]
\leftline{\includegraphics[width=\columnwidth]{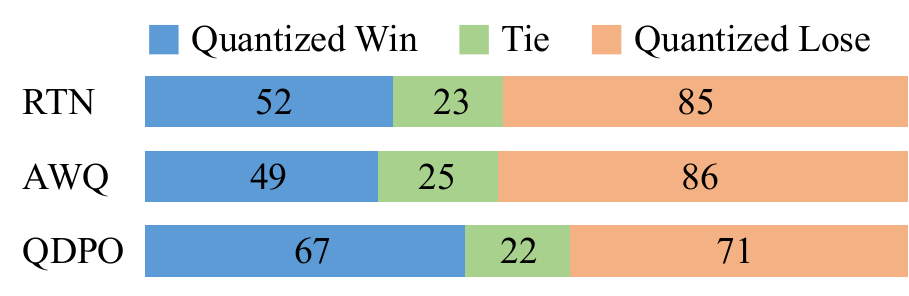}}
\caption{Vicuna-Eval results on Mi:dm.}
\label{fig:vicuna_eval}
\end{figure}

% To more consistently observe the performance improvements of models applying QDPO, we extend our experiments to Vicuna-Eval, a widely used performance evaluation for language models~\cite{lin2023awq,lee2023owq}. We have the answers of 80 samples evaluated by GPT-4 for the Mi:dm. As shown in Fig.~\ref{fig:vicuna_eval}, it can be seen that models with QDPO applied exhibit the highest wins and the lowest losses, demonstrating a lose-rate of 50\%, which indicates a near recovery of the language capabilities of the baseline model.
Since Vicuna-Eval is a widely used benchmark for evaluating conversational abilities, we further employ it for evaluating QDPO. We take Mi:dm as a target language model to apply different quantization methods and evaluate its performance on 80 questions by GPT-4. As shown in Fig.~\ref{fig:vicuna_eval}, it can be seen that models with QDPO applied exhibit the highest wins and the lowest losses, demonstrating a lose-rate of 50\%, which indicates a near recovery of the language capabilities of the baseline model.

\subsection{Experimental Results: FLASK}
We use the FLASK benchmark on Mi:dm to verify how the proposed method enhances the fine-grained skills of the language model. Fig.~\ref{fig:flask} shows the relative performance of different quantized LLMs across the 12 fine-grained skills. RTN significantly diminishes certain capabilities of the model, while AWQ and KD slightly improve performance toward the 16-bit weight baseline. In contrast, QDPO shows a significant enhancement in most skills; in particular, QDPO significantly improves metacognition skills, whereas RTN and AWQ significantly fall short. (Details on skill-wise analysis can be found in \ref{appnedix:flask-skill-examples}.) Overall, QDPO achieves the abilities closest to the 16-bit weight baseline, showcasing the effectiveness of its alignment objective in recovering conversational skills. 
% For example, as shown in Fig.~\ref{fig:flask_metacognition1}, RTN opts for "<[!newline]>" instead of ":", leading to subsequent generations consisting solely of simple listings, and it can be observed that sentences become repetitive as they lengthen. In contrast, models applying QDPO follow the baseline by providing explanations for each item.

\begin{table}[t]
\resizebox{\columnwidth}{!}{%
\begin{tabular}{c|cc|cc|c}
\Xhline{3\arrayrulewidth}
Method   & CSQA  & MMLU  & DROP  & BBH   & MT-bench \\ \midrule
Baseline & 75.16 & 46.55 & 24.95 & 34.23 & 4.07     \\ \midrule
RTN      & 73.87 & 42.46 & 21.81 & 32.57 & 3.52     \\
RTN+QDPO & 73.11 & 42.69 & 21.50 & 32.05 & 3.96     \\ \midrule
AWQ      & 74.75 & 45.06 & 24.07 & 32.63 & 3.75     \\
AWQ+QDPO & 74.29 & 44.99 & 24.12 & 32.76 & 3.87     \\ \Xhline{3\arrayrulewidth}
\end{tabular}%
}
% \vspace{-2mm}
\caption{W4A16 inference results on conventional benchmarks (Mi:dm).}
% \vspace{-3mm}
\label{tab:benchmarks}
\end{table}
\begin{figure}[t]
\centerline{\includegraphics[width=\columnwidth]{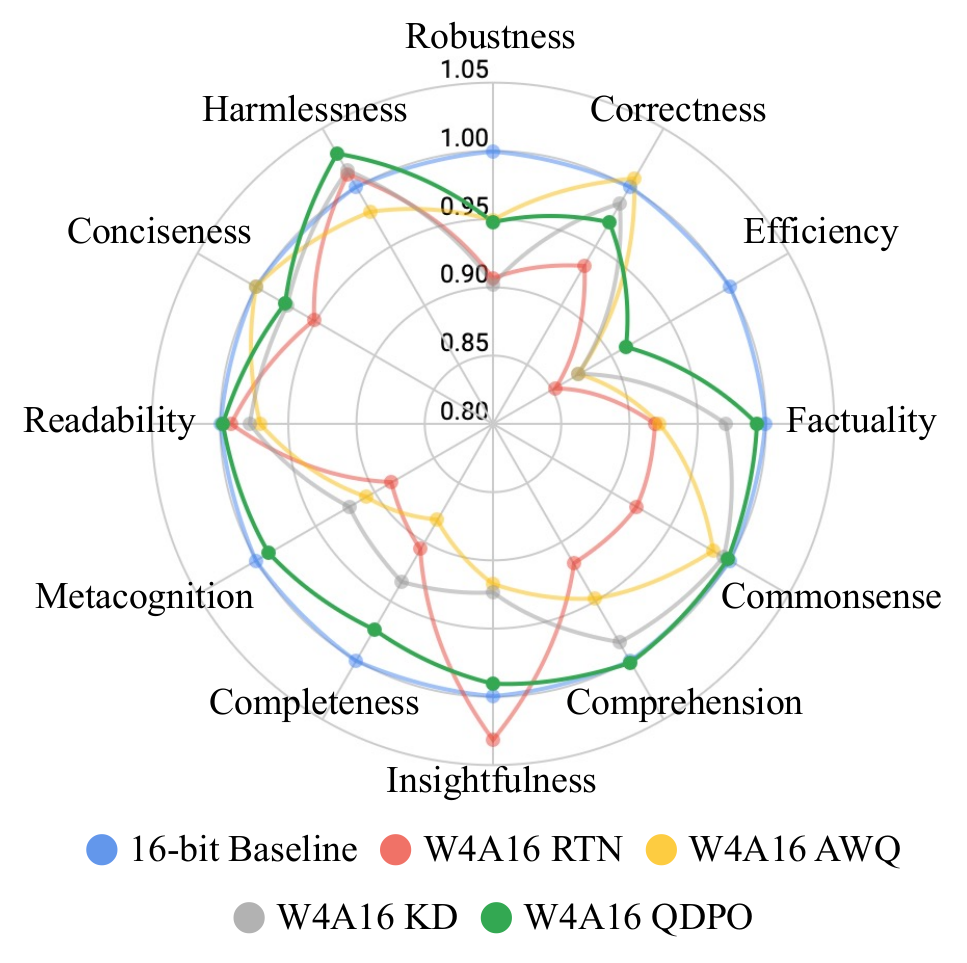}}
\caption{Performance relative to baseline on FLASK. Absolute performance results can be found in Table~\ref{tab:flask}.}
\vspace{-3mm}
\label{fig:flask}
\end{figure}

\subsection{Ablation Studies}
\label{sec:ablation}
We further conduct ablation studies to provide insights on QDPO for improving the conversational skills of quantized LLMs. 

\textbf{Conversation Abilities vs. Task Accuracy. }
As discussed, QDPO has the particular role of aligning the quantized LLMs to the 16-bit weight baseline. What is the impact of this alignment on the task-specific performance of LLMs? To answer this question, we further evaluate the quantized LLMs on well-known benchmarks that test the task-specific capability of language models. In particular, Common Sense Question Answering (CSQA)~\cite{talmor2019commonsenseqa} and Massive Multitask Language Understanding (MMLU)~\cite{mmlu} assess the models' reasoning and multitask-solving abilities through multiple-choice questions. Furthermore, DROP~\cite{dua-etal-2019-drop} and BBH~\cite{srivastava2023beyond_bbh} evaluate the problem-solving abilities of instruction-tuned models in logic and math. (Details on the task-specific benchmarks are in \ref{appnedix:task-specific-benchmarks}.) Table~\ref{tab:benchmarks} compares the task accuracy (CSQA, MMLU, DROP, BBH) as well as the conversational abilities (MT-Bench) on the quantized LLM with and without QDPO. RTN suffers degradation on the task accuracy as well as the conversational abilities. Interestingly, AWQ significantly improves task accuracy while its conversational abilities are marginally improved. Meanwhile, QDPO improves conversational ability while mostly preserving task accuracy, showcasing its usefulness.

% To assess the reasoning capabilities of Large Language Models (LLMs), benchmarks such as Common Sense Question Answering (CSQA)~\cite{talmor2019commonsenseqa} and MMLU~\cite{mmlu} have been widely utilized. CSQA assesses models' reasoning abilities through multiple-choice questions, while MMLU verifies models' multitask solving capabilities across 57 different tasks with multiple-choice questions. Recently, benchmarks like DROP~\cite{dua-etal-2019-drop} and BBH~\cite{srivastava2023beyond_bbh} have been used to evaluate the problem-solving abilities of instruction-tuned models, testing skills in logic and math. Additionally, the Helpful, Honest, and Harmless (HHH)~\cite{askell2021general_hhh} benchmark is widely used to assess the extent to which these models are safe or beneficial to humans. In our experiments, we measure zero-shot CSQA benchmark and average across five tasks (WinoGrande~\cite{sakaguchi2019winogrande}, COPA~\cite{copa}, PIQA~\cite{bisk2019piqa}, BoolQ~\cite{clark-etal-2019-boolq}, HellaSwag~\cite{hellaswag}). 

\textbf{Conversation Abilities vs. Perplexity. }
% \usepackage{graphicx}
\begin{table}[t]
\centering
\resizebox{\columnwidth}{!}{%
\begin{tabular}{c|c|c|cc}
\Xhline{3\arrayrulewidth}
Languange & Model & Method & PPL ↓ & Lose-rate ↓ \\ \midrule
\multirow{8}{*}{English} & \multirow{4}{*}{Mi:dm} & 16-bit Baseline & 13.12 & - \\ \cline{3-5} 
 &  & RTN & 15.16 & 0.69 \\
 &  & AWQ & \textbf{14.23} & 0.58 \\
 &  & QDPO & 15.55 & \textbf{0.40} \\ \cline{2-5} 
 & \multirow{4}{*}{Vicuna} & 16-bit Baseline & 6.78 & - \\ \cline{3-5} 
 &  & RTN & 7.53 & 0.60 \\
 &  & AWQ & \textbf{7.34} & \textbf{0.44} \\
 &  & QDPO & 7.36 & \textbf{0.44} \\ \midrule
\multirow{4}{*}{Korean} & \multirow{4}{*}{Mi:dm} & 16-bit Baseline & 5.71 & - \\ \cline{3-5} 
 &  & RTN & 6.52 & 0.60 \\
 &  & AWQ & \textbf{5.97} & 0.62 \\
 &  & QDPO & 6.56 & \textbf{0.55} \\ \Xhline{3\arrayrulewidth}
\end{tabular}%
}
\caption{Perplexity (PPL) evaluation and lose-rate from MT-bench for W4A16 quantized LLMs.}
\vspace{-2mm}
\label{tab:ppl}
\end{table}
Perplexity is a key metric for evaluating language models, as it measures the exponentiated average negative log probability of predicted word sequences. We examine whether the enhanced conversational capabilities through QDPO are also reflected in perplexity by comparing the perplexity and the lose-rate on the MT-bench. We measure perplexity using Wikitext-2~\cite{merity2016pointerwiki} for English and Korean textbooks\footnote{\url{https://huggingface.co/datasets/maywell/korean_textbooks}} dataset for Korean. As shown in Table~\ref{tab:ppl}, RTN significantly increases perplexity across all models. While AWQ decreases perplexity in all models, it does not guarantee an improvement in conversational ability. For example, in Mi:dm’s Korean benchmark, AWQ significantly reduces perplexity by 0.55 compared to RTN, yet the lose-rate increases by 2\%. On the other hand, QDPO significantly enhances conversational ability, even though it does not achieve as low a perplexity as the baseline. We believe that the discrepancy between perplexity and conversational ability stems from the difficulty of using next-word prediction perplexity on reference text to capture the impact of flipped tokens in an auto-regressive generation. From our observation, these tokens significantly contribute to sentence variation, as discussed in Sec.~\ref{sec:conversation:breakdown}.

\textbf{QDPO vs. Beam Search. }
% \usepackage{multirow}
% \usepackage{graphicx}
\begin{table}[t]
\resizebox{\columnwidth}{!}{%
\begin{tabular}{c|c|ccc|c}
\Xhline{3\arrayrulewidth}
Method               & Number of Beams & Win         & Tie         & Lose        & Lose-rate ↓ \\ \midrule
RTN                  & 1               & 24          & 6           & 66          & 0.69      \\ \midrule
\multirow{3}{*}{AWQ} & 1               & 28          & 9           & 52          & 0.58      \\
                     & 3               & 38          & 9           & 50          & 0.52      \\
                     & 5               & 35          & 10          & 61          & 0.58      \\ \midrule
QDPO                 & 1               & 53 & 44 & 14 & \textbf{0.40}      \\ \Xhline{3\arrayrulewidth}
\end{tabular}%
}
% \vspace{-2mm}
\caption{Impact of decoding strategy.}
\vspace{-2mm}
\label{tab:beam_search}
\end{table}
Beam search \cite{graves2012sequence} generates higher probability outcomes by considering multiple generation possibilities simultaneously. Therefore, even if the quantized model makes different judgments from the baseline, which significantly influences sentence generation, there is still a possibility of generating outcomes without issues in overall probability. We aim to observe how decoding strategies affect quantized generation across three beam sizes (1, 3, 5). Table~\ref{tab:beam_search} shows the results of the MT-Bench pairwise comparison according to decoding strategies. In the case of a beam size of 3, generating a more diverse range of sentences slightly reduces the lose-rate, yet it still exhibits many losses, and increasing the beam size further does not fundamentally solve the problem, as it also increases defeats. In contrast, QDPO demonstrates a lower lose-rate by creating models that are robust to quantization.

\section{Conclusion}
\label{sec:conclusion}

In this work, we address the conversational abilities of quantized LLM-based chatbots. After identifying token-flipping as a crucial factor for degraded text generation quality, we propose a novel quantization-aware direct preference optimization (QDPO) method that effectively aligns quantized and full-precision LLMs, enhancing conversational performance. Tested across multiple languages and models, QDPO outperforms traditional fine-tuning techniques, setting a new benchmark for conversational chatbot development. 

\section*{Acknowledgement}
% This work was supported by Institute of Information \& communications Technology Planning \& Evaluation (IITP) (2024-RS-2023-00253914, under the artificial intelligence semiconductor support program to nurture the best talents, and 2022-0-00971, Logic Synthesis for NVM-based PIM Computing Architecture) and National Research Foundation of Korea (NRF) (No. RS-2023-00260527) grant funded by the Korea government (MSIT). This work was also supported by Artificial Intelligence Industrial Convergence Cluster Development Project, funded by the Ministry of Science and ICT (MSIT, Korea) and Gwangju Metropolitan City.
This work was partly supported by Institute of Information \& communications Technology Planning \& Evaluation (IITP) grant funded by the Korea government (MSIT) (No. RS-2020-II201373, Artificial Intelligence Graduate School Program (Hanyang University), and 2022-0-00971, Logic Synthesis for NVM-based PIM Computing Architecture) and National Research Foundation of Korea (NRF) (No. RS-2023-00260527). This work was also partly supported by Artificial Intelligence Industrial Convergence Cluster Development Project, funded by the Ministry of Science and ICT (MSIT, Korea) \& Gwangju Metropolitan City, and the research fund of Hanyang University (HY-201900000002966).
\section*{Limitations}
Our objective is to align the language capabilities of a baseline model distorted by quantization error through DPO. We focus on exploring scenarios where quantization error does not completely ruin conventional benchmarking performance yet introduces subtle differences in language capabilities that are perceptible to humans. Hence, we do not address situations where large quantization errors significantly degrade model performance, nor do we deal with cases using fine-grained quantization where quantization error is minimal. However, from a practical standpoint, the challenge of reducing the inference cost of LLMs by transitioning to lower bit-precision is necessary, and this process should consider various techniques, including group quantization. Additionally, since our approach involves aligning the baseline model with a relatively minimal training process, there are limitations in utilizing extensive datasets. Nonetheless, the impact of different datasets when aligning the baseline model with a limited number of bits remains an intriguing topic.

% 1) QDPO - different conversational dataset for preference data generation

% 2) QDPO with different precision for $pi_q$ 

\bibliography{anthology,custom}
\bibliographystyle{acl_natbib}

\appendix

% \clearpage
\section{Appendix}
\label{sec:appendix}

%%%%%%%%%%%%%%%%%%%%%%% Experiments
\subsection{Experimental Details}
\label{appendix:training_details}
\textbf{\quad PTQ Calibration Settings.} For PTQ calibration, we use the widely utilized method AWQ~\cite{lin2023awq}, with the calibration set consisting of 64 samples randomly extracted from the C4~\cite{c4} dataset. We apply channel-wise quantization and do not consider fine-grained quantization (e.g. group quantization) to better observe the impact of quantization on the LLM's conversational abilities.

\textbf{Knowledge Distillation Settings.} For KD setting, we follow KD method introduced in LLM-QAT~\cite{liu2023llmqat}, excluding the data curation process. To facilitate a fair comparison with QDPO, we extracted 50,000 prompts from the Anthropic Helpful and Harmless dialogue dataset~\cite{bai2022training}, and set the learning rate to 3e-6.

\textbf{Training Settings.} In our QDPO experiments, similar to the KD, we sample 50,000 prompts in English from the Anthropic Helpful and Harmless dialogue dataset and 21,155 prompts in Korean from the KoAlpaca\footnote{\url{ https://huggingface.co/datasets/beomi/KoAlpaca-v1.1a}} dataset. We collect responses using both the full-precision policy ($\pi_\text{fp}$) and the quantized policy  ($\pi_\text{q}$) to construct a preference pair dataset. The learning rate is set to 3e-6.\\
% LR, Epoch, Dataset? ...

%%%%%%%%%%%%%%%%%%%%%%% Section 3 Appendix
\subsection{More Details on Breakdown Analysis}
To separate the cause of errors in text generation, we employ the following steps:
\label{appendix:breakdown}
\begin{itemize}
    \item We provide the same input to both the baseline and quantized models, then observe the first 100 generations and find the timestep at which the first different token is generated between the two models. We dump these differently generated tokens, which are flipped tokens.
    \item We dump the KV cache of both the baseline and quantized models until timestep. This is facilitated easily through HuggingFace\footnote{\url{ https://github.com/huggingface/transformers}}'s \texttt{past\_key\_values} argument.
    \item Based on the dumped flipped tokens and KV cache, we observe additional generations with either the baseline or quantized model, depending on our purpose in reflecting $\text{Q}_{\text{error}}$.
\end{itemize}
\begin{figure}[t]
\centerline{\includegraphics[width=\columnwidth]{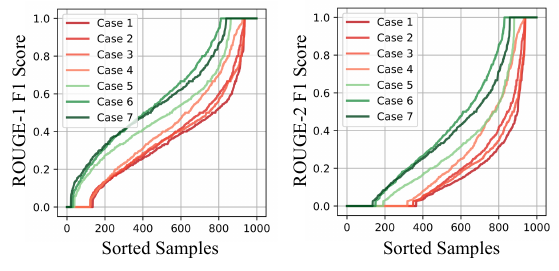}}
\caption{ROUGE-1 and ROUGE-2 scores for Fig~\ref{fig:breakdown}(c) (W4A16 RTN).}
\label{fig:rouge-1-2}
\end{figure}

\begin{figure}[t]
\centerline{\includegraphics[width=\columnwidth]{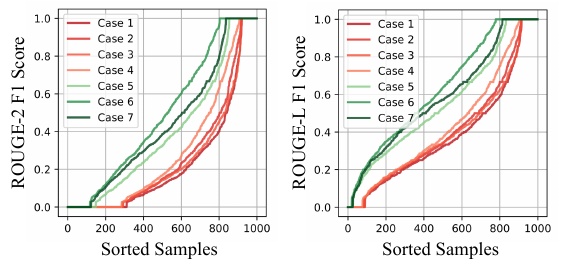}}
\caption{More results of breakdown analysis from Sec.~\ref{sec:conversation:breakdown} on KD-based QAT (W4A16).}
\label{fig:rouge-kd}
\end{figure}

\begin{figure}[t]
\centerline{\includegraphics[width=\columnwidth]{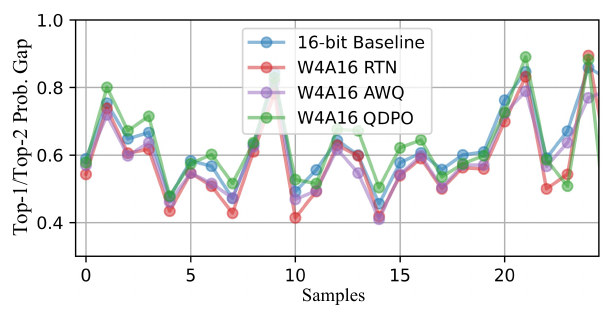}}
\caption{The average gap between the Top-1 and Top-2 tokens in AWQ shows a closer probability difference to the baseline compared to RTN, thanks to the reduction of quantization error. However, AWQ still exhibits a lower gap than the baseline model.}
\label{fig:prob_gap_awq}
\end{figure}

% %%%%%%%%%%%%%%%%%%%%%%% Section 3 Appendix
\subsection{QDPO’s Compatibility with Existing Techniques}

\quad \textbf{QDPO on RLHF-tuned Models.} We conduct additional experiments to investigate whether QDPO can serve as a complementary approach to recover conversational in quantized RLHF-tuned models, such as LLaMA2-Chat~\cite{touvron2023llama2} in Table~\ref{tab:appendix:llama2-chat}, In LLaMA2-Chat, which has improved conversational abilities through reflecting human preferences via RLHF, W4A16 RTN exhibits a 50\% lose-rate compared to the baseline model. However, QDPO demonstrates further recovery of conversational ability. With more aggressive quantization at 3-bit, a clearer trend is observed. RTN experiences a rapid decline in conversational ability. In contrast, QDPO significantly reduces the lose-rate by restoring the conversational ability of the baseline model. This demonstrates that QDPO effectively enhances the conversational ability of a quantized RLHF-tuned model, indicating that it is a method compatible with existing RLHF.
\begin{table}[t]
\resizebox{\columnwidth}{!}{%
\begin{tabular}{c|c|ccc|c}
\Xhline{3\arrayrulewidth}
Bit-precision & Method & Win & Tie & Lose & Lose-rate ↓ \\ \midrule
\multirow{2}{*}{W4A16} & RTN & 40 & 20 & 60 & 50.00\% \\ 
 & QDPO & 38 & 26 & 55 & \textbf{46.22\%} \\ \midrule
\multirow{2}{*}{W3A16g128} & RTN & 29 & 18 & 76 & 61.79\% \\ 
 & QDPO & 35 & 22 & 65 & \textbf{53.28\%} \\  \Xhline{3\arrayrulewidth}
\end{tabular}%
}
\caption{QDPO on RLHF-tuned model (LLaMA2-Chat 7B). g128 denotes fine-grained quantization with groupsize=128.}
\label{tab:appendix:llama2-chat}
\end{table}

\textbf{QDPO with Memory-Efficient Fine-Tuning Method.} We extend our experiments with QDPO training using LoRA~\cite{hu2022lora}. Following the approach in QLoRA~\cite{dettmers2023qlora}, we keep the quantized base weights frozen and train only the high-precision adapter. To ensure a fair comparison with other methods, we utilize INT4 for the base weights instead of NF4~\cite{dettmers2023qlora}. The adapter rank and $\alpha$ used in this experiment is 64.
As shown in Table~\ref{tab:appendix:lora}, QDPO with LoRA significantly enhances the conversational ability of quantized LLMs, achieving levels nearly identical to those of QDPO. Moreover, QDPO with LoRA reduces the required memory by keeping the quantized base weights fixed and significantly decreases the number of training parameters by only training the LoRA adapter. These results suggest that QDPO serves as a complementary method that can be utilized alongside other techniques.
\begin{table*}[t]
\resizebox{\textwidth}{!}{%
\begin{tabular}{c|ccc|c|c|c|c}
\Xhline{3\arrayrulewidth}
Method & Win & Tie & Lose & Lose-rate ↓ & \# Trainable params & Required Memory for Training$^\ast$ & Inference bit-width \\ \midrule
RTN & 24 & 6 & 66 & 0.69 & - & - & W4A16 \\ \midrule
AWQ & 28 & 9 & 52 & 0.58 & - & - & W4A16 \\ \midrule
KD & 31 & 16 & 52 & 0.53 & 7.02 B & 56.16 GB & W4A16 \\ \midrule
QDPO & 53 & 14 & 44 & \textbf{0.40} & 7.02 B & 56.16 GB & W4A16 \\ \midrule
QDPO+LoRA & 48 & 14 & 46 & 0.43 & \textbf{1.33 B} & \textbf{14.15 GB} & W16A16 \\ \Xhline{3\arrayrulewidth}
\end{tabular}%
}
\caption{QDPO with LoRA. $^\ast$We only measure required memory for $\pi_\theta$ during training (Mi:dm-7B, MT-bench pairwise comparison).}
\label{tab:appendix:lora}
\end{table*}

%%%%%%%%%%%%%%%%%%%%%%% MT-bench example
\subsection{MT-Bench Evaluation Metrics}

% \usepackage{multirow}
% \usepackage{graphicx}
\begin{table}[t]
\resizebox{\columnwidth}{!}{%
\begin{tabular}{l|c|c|ccc|c}
\Xhline{3\arrayrulewidth}
Lang.                & Model                   & Method & Win & Tie & Lose & Lose-Rate \\ \midrule
\multirow{7}{*}{Eng} & \multirow{4}{*}{Mi:dm}   & RTN    & 18  & 103 & 55   & 0.31      \\
                     &                         & AWQ    & 23  & 110 & 46   & 0.26      \\
                     &                         & KD     & 22  & 115 & 40   & 0.23      \\
                     &                         & QDPO   & 43  & 118 & 38   & 0.19      \\ \cline{2-7} 
                     & \multirow{3}{*}{Vicuna} & RTN    & 20  & 95  & 61   & 0.35      \\
                     &                         & AWQ    & 31  & 107 & 46   & 0.25      \\
                     &                         & QDPO   & 33  & 113 & 44   & 0.23      \\ \midrule
\multirow{3}{*}{Kor} & \multirow{3}{*}{Mi:dm}   & RTN    & 20  & 103 & 45   & 0.27      \\
                     &                         & AWQ    & 20  & 111 & 37   & 0.22      \\
                     &                         & QDPO   & 39  & 109 & 40   & 0.21      \\ \Xhline{3\arrayrulewidth}
\end{tabular}%
}
\caption{Pairwise comparison results of MT-Bench following original metric of \cite{zheng2023judging}.}
\label{tab:mtbench_pair_origin}
\end{table}
\label{appendix:mt-bench-score}
\textbf{\quad Pairwise Comparison.} In pairwise comparison within MT-Bench for 80 samples, GPT-4 evaluates which model provides better responses between the two models. However, due to most LLMs' tendency to prefer the first position~\cite{zheng2023judging}, the evaluation occurs twice in reversed order, counting victories only if one model wins in both cases. If judgments reverse or both evaluations result in ties, it counts as an actual tie. We observe that GPT-4 frequently evaluates "tie" more often than usual in comparisons between different models in MT-Bench. This increased frequency of ties is because our study focuses on comparing similar models (the baseline model and the quantized model). We find that cases evaluated as a tie in both positions present many obstacles to the evaluation we desire for judging alignment. For example, as shown in Fig.~\ref{fig:tie_example1} and Fig.~\ref{fig:tie_example2}, when the baseline model provides an incorrect answer and the quantized model also offers a wrong answer (but a different response), GPT-4 provides a "tie" because they are both incorrect. However, this "tie" does not reflect our goal of assessing "how well two models are aligned." Therefore, we evaluate a tie only in cases where win/lose changes due to swapping positions, causing GPT-4 confusion. We find this evaluation method to be the most consistent with the results of other benchmarks, like Vicuna-Eval~\cite{vicuna2023}. Results obtained using the original evaluation criteria of MT-Bench is in Table~\ref{tab:mtbench_pair_origin}.

\textbf{Single-Answer Grading.} In single-answer grading, we directly request GPT-4 to assign scores of up to 10 points. While this approach may not be as nuanced as pairwise comparison in model comparisons, it enables observation of how quantization induces changes in specific categories where the model has strengths and weaknesses by measuring absolute scores by category.

\subsection{Detailed Analysis by Catagory in MT-Bench}
\label{appnedix:mt-bench-category-analysis}
As shown in Table~\ref{tab:mtbench_single_midm}, QDPO improves overall capability compared to other methods. However, we observe that in some categories, QDPO scores are lower than the AWQ model. We conducted a more detailed observation of GPT-4's evaluations in areas where QDPO exhibits lower performance. Interestingly, as depicted in Fig.~\ref{fig:mt_single_writing}, we can see that QDPO fails in cases where RTN already provides a good response and receives a high score, almost the same as the baseline generation. We believe this might be the case because QDPO is trained to reject sentences generated by the quantized model, which can lead to optimization challenges in such situations. Additional examples are present in Fig.~\ref{fig:mt_single_extraction1}.
% In cases like those shown in Fig.~\ref{fig:mt_single_extraction2}, where all except RTN produced the same answer but with different scores, QDPO still attempts to generate creations that follow the baseline model.

\subsection{Skill-wise Analysis in FLASK}
\label{appnedix:flask-skill-examples}

% \usepackage{multirow}
% \usepackage{graphicx}
\begin{table}[t]
\resizebox{\columnwidth}{!}{%
\begin{tabular}{c|c|cccc}
\Xhline{3\arrayrulewidth}
\multirow{2}{*}{Catergory} & \multirow{2}{*}{\begin{tabular}[c]{@{}c@{}}16-bit\\ Baseline\end{tabular}} & \multicolumn{4}{c}{W4A16}                                          \\
                           &                                                                            & RTN            & AWQ            & KD    & \multicolumn{1}{l}{QDPO} \\ \midrule
Robustness                 & 2.029                                                                      & 1.839          & \textbf{1.927} & 1.830 & 1.924                    \\
Correctness                & 2.237                                                                      & 2.087          & \textbf{2.254} & 2.206 & 2.172                    \\
Efficiency                 & 2.333                                                                      & 1.988          & 2.036          & 2.036 & \textbf{2.129}           \\
Factuality                 & 2.709                                                                      & 2.487          & 2.497          & 2.631 & \textbf{2.691}           \\
Commonsense                & 2.965                                                                      & 2.735          & 2.925          & 2.953 & \textbf{2.961}           \\
Comprehension              & 2.874                                                                      & 2.639          & 2.725          & 2.831 & \textbf{2.879}           \\
Insightfulness             & 2.268                                                                      & \textbf{2.339} & 2.079          & 2.095 & 2.246                    \\
Completeness               & 2.858                                                                      & 2.587          & 2.518          & 2.666 & \textbf{2.784}           \\
Metacognition              & 2.891                                                                      & 2.562          & 2.625          & 2.663 & \textbf{2.863}           \\
Readability                & 4.079                                                                      & 4.047          & 3.956          & 3.989 & \textbf{4.070}           \\
Conciseness                & 3.881                                                                      & 3.695          & \textbf{3.886} & 3.782 & 3.785                    \\
Harmlessness               & 4.447                                                                      & 4.500          & 4.355          & 4.512 & \textbf{4.575}           \\ \midrule
Average                    & 2.964                                                                      & 2.792          & 2.815          & 2.849 & \textbf{2.923}           \\ \Xhline{3\arrayrulewidth}
\end{tabular}%
}
\caption{FLASK score per skill.}
\label{tab:flask}
\end{table}
We aim to investigate how QDPO recovers skills that, according to FLASK's fine-grained categorization, significantly underperform in RTN and AWQ compared to the baseline. As shown in Fig.~\ref{fig:flask_metacognition1}, RTN opts for "<[!newline]>" instead of ":", leading to subsequent generations consisting solely of simple listings, and it can be observed that sentences become repetitive as they lengthen. In contrast, models applying QDPO follow the baseline by providing explanations for each item.

\subsection{Details of Task-Specific Benchmarks}
\label{appnedix:task-specific-benchmarks}
To assess the reasoning capabilities of Large Language Models (LLMs), benchmarks such as Common Sense Question Answering (CSQA)~\cite{talmor2019commonsenseqa} and MMLU~\cite{mmlu} have been widely utilized. CSQA assesses models' reasoning abilities through multiple-choice questions, while MMLU verifies models' multitask-solving capabilities across 57 different tasks with multiple-choice questions. Recently, benchmarks like DROP~\cite{dua-etal-2019-drop} and BBH~\cite{srivastava2023beyond_bbh} have been used to evaluate the problem-solving abilities of instruction-tuned models, testing skills in logic and math. Additionally, the Helpful, Honest, and Harmless (HHH)~\cite{askell2021general_hhh} benchmark is widely used to assess the extent to which these models are safe or beneficial to humans. In our experiments, we measure zero-shot CSQA benchmark and average across five tasks (WinoGrande~\cite{sakaguchi2019winogrande}, COPA~\cite{copa}, PIQA~\cite{bisk2019piqa}, BoolQ~\cite{clark-etal-2019-boolq}, HellaSwag~\cite{hellaswag}).

\subsection{Proof of Theorem.~\ref{thm:D_QDPO}}
\setcounter{theorem}{0}
\label{proof2}
\begin{theorem}
     For any response \(y\) in the set of all possible responses \(Y\), if \(y_1 = \arg\max_{y \in Y} \pi_\text{fp}(y|x)\) and \(y_2 = \arg\max_{y \in Y} \pi_\text{q}(y|x)\), then it is guaranteed that \(p^*(y_1 \succ y_2) \geq p^*(y_2 \succ y_1)\).
\end{theorem}
\begin{proof}
The definition of \(\arg\max\) ensures that for all \(y \in Y\), \(\pi_\text{fp}(y_1|x) \geq \pi_\text{fp}(y|x)\) and \(\pi_\text{q}(y_2|x) \geq \pi_\text{q}(y|x)\) holds true. Consequently, this implies \(\pi_\text{fp}(y_1|x) \geq \pi_\text{fp}(y_2|x)\) and \(\pi_\text{q}(y_2|x) \geq \pi_\text{q}(y_1|x)\).\\
Substituting eq.~(\ref{eq:opt_reward}) into eq.~(\ref{eq:BT}) we obtain:
\begin{equation}
\begin{aligned}
p^*(y_1 &\succ y_2|x)\\
&=\frac{1}{1 + \exp \left( \beta \log \frac{\pi_\text{fp}(y_2|x)}{\pi_{q}(y_2|x)} - \beta \log \frac{\pi_\text{fp}(y_1|x)}{\pi_{q}(y_1|x)} \right)}\\
&= \sigma \left( \beta \log \frac{\pi_\text{fp}(y_1|x)}{\pi_{q}(y_1|x)} - \beta \log \frac{\pi_\text{fp}(y_2|x)}{\pi_{q}(y_2|x)} \right)\\
%&=\sigma \left( \beta \left( \log \frac{\pi_\text{fp}(y_1|x)}{\pi_\text{fp}(y_2|x)}- \log \frac{\pi_{q}(y_1|x)}{\pi_{q}(y_2|x)} \right) \right)\\
&=\sigma \left( \beta \left( \log \frac{\pi_\text{fp}(y_1|x)}{\pi_\text{fp}(y_2|x)}- \log \frac{\pi_\text{q}(y_1|x)}{\pi_\text{q}(y_2|x)} \right) \right)
\end{aligned}
\end{equation}
% \begin{multline}
% \label{eq:proof}
% p^*(y_1 \succ y_2|x) = \sigma \left( \beta \left( \log \frac{\pi_\text{fp}(y_1|x)}{\pi_\text{fp}(y_2|x)}\right.\right.\\ \left.\left.- \log \frac{\pi_\text{q}(y_1|x)}{\pi_\text{q}(y_2|x)} \right) \right),
% \end{multline}
$\log \frac{\pi_\text{fp}(y_1|x)}{\pi_\text{fp}(y_2|x)}- \log \frac{\pi_\text{q}(y_1|x)}{\pi_\text{q}(y_2|x)} $ and \(\beta\) is positive, it follows that \(p^*(y_1 \succ y_2|x) \geq 0.5\). Consequently, this implies that \(p^*(y_1 \succ y_2) \geq p^*(y_2 \succ y_1)\).
\end{proof}
% Substituting eq.~(\ref{eq:opt_reward}) into eq.~(\ref{eq:BT}) we obtain :

\subsection{Generation Examples}
\label{appendix:examples}

Fig.~\ref{fig:example1} demonstrates a decline in language model performance due to the generation of different tokens compared to the baseline. The baseline model selects ``Wear" following ``1.", whereas the PTQ model, experiencing a change in probability ranking, chooses ``Always." The PTQ model then repeats this word, leading to expressions that feel awkward to humans. On the other hand, QDPO recovers the probabilities similar to the baseline model, thereby continuing with the natural generation.
\begin{figure}[ht]
\centerline{\includegraphics[width=\columnwidth]{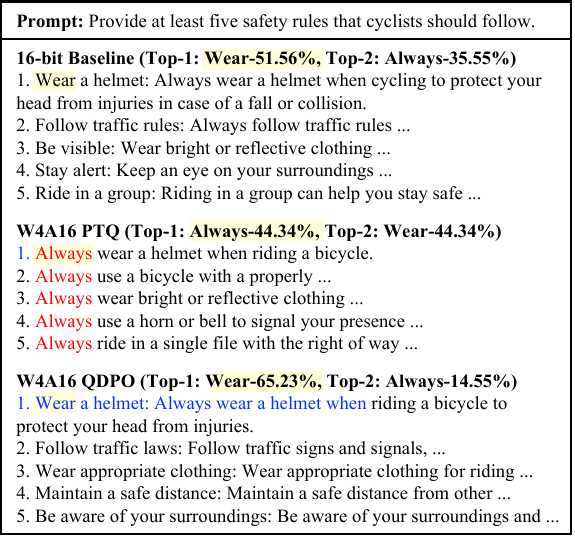}}
\caption{Generation example of 16-bit and W4A16 inference (Model: Mi:dm-7B).}
\label{fig:example1}
\end{figure}
\begin{figure*}[ht]
% \centerline{\includegraphics[width=\columnwidth]{Figures/example2.pdf}}
\centerline{\includegraphics[width=\textwidth]{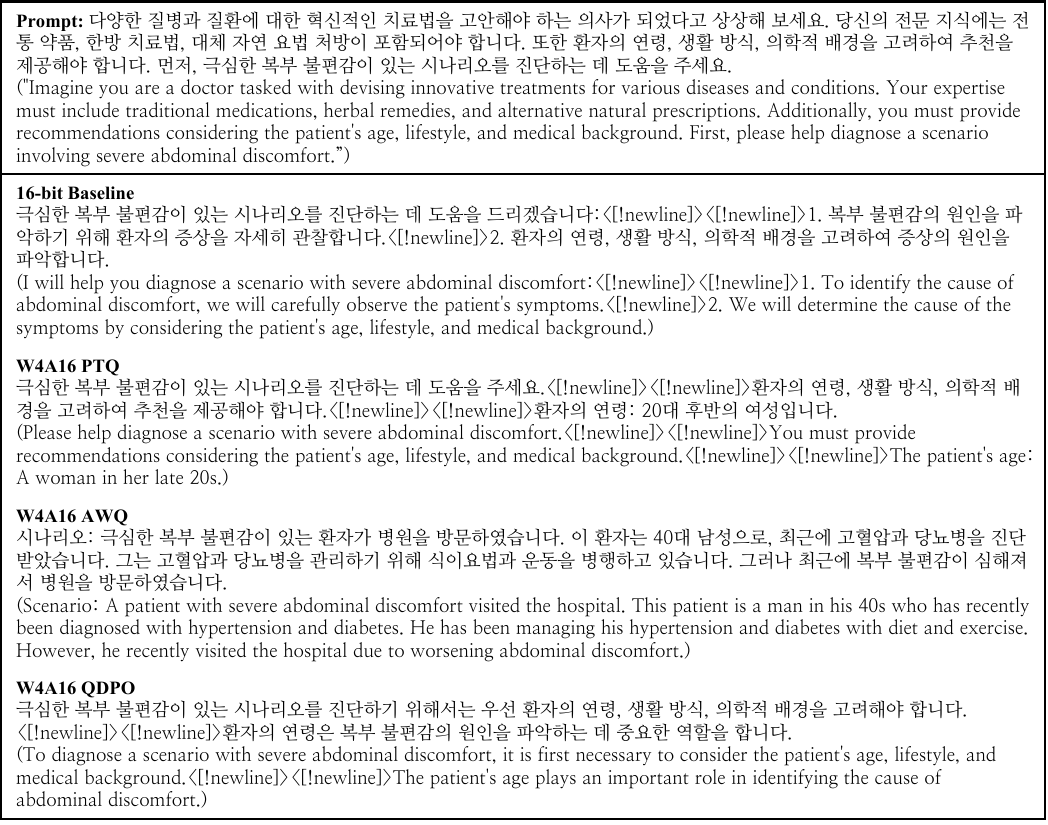}}
\caption{Generation example of 16-bit and W4A16 inference (Korean).}
\label{fig:example2}
\end{figure*}

\begin{figure*}[ht]
\centerline{\includegraphics[width=\textwidth]{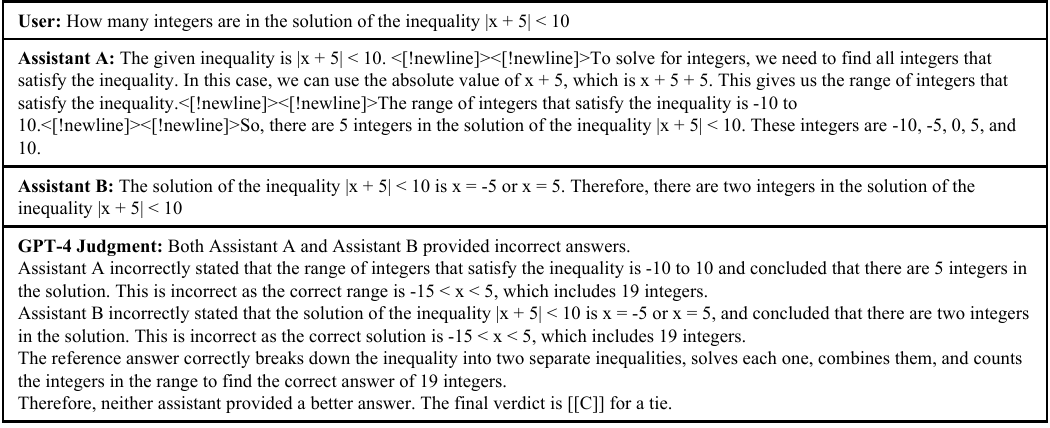}}
\caption{Tie example in MT-Bench pairwise comparison.}
\label{fig:tie_example1}
\end{figure*}

\begin{figure*}[ht]
\centerline{\includegraphics[width=\textwidth]{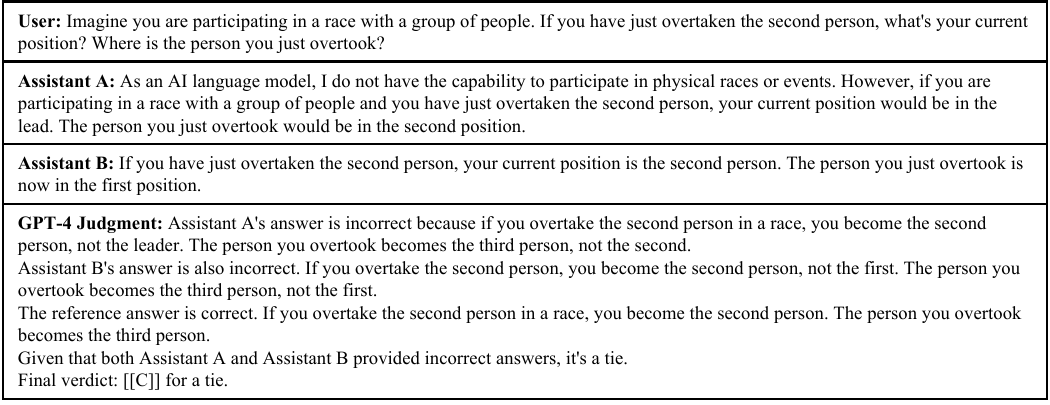}}
\caption{Tie example in MT-Bench pairwise comparison.}
\label{fig:tie_example2}
\end{figure*}

\begin{figure*}[ht]
\centerline{\includegraphics[width=\textwidth]{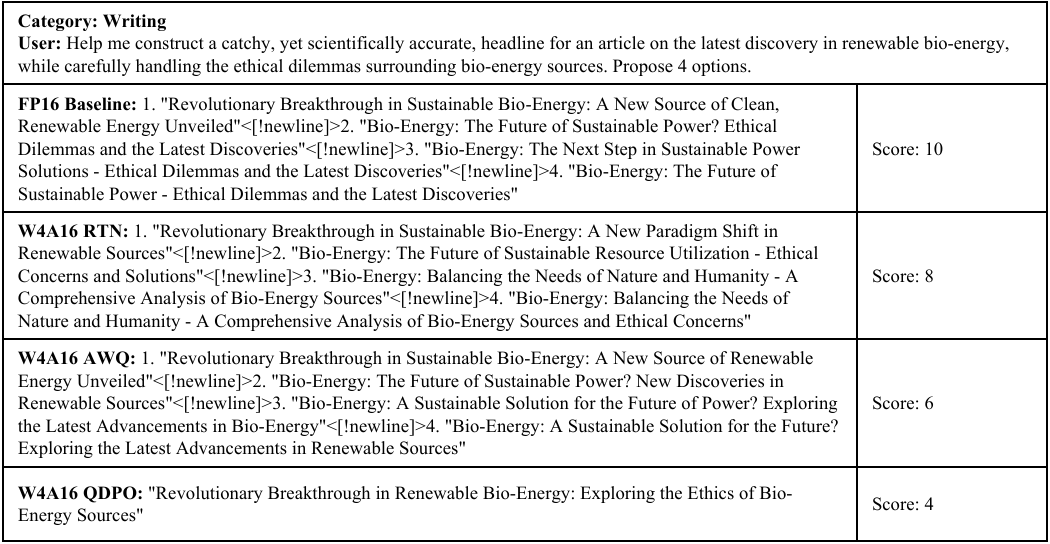}}
\caption{MT-Bench single-grading judgment example (Category: Writing).}
\label{fig:mt_single_writing}
\end{figure*}
\begin{figure*}[ht]
\centerline{\includegraphics[width=\textwidth]{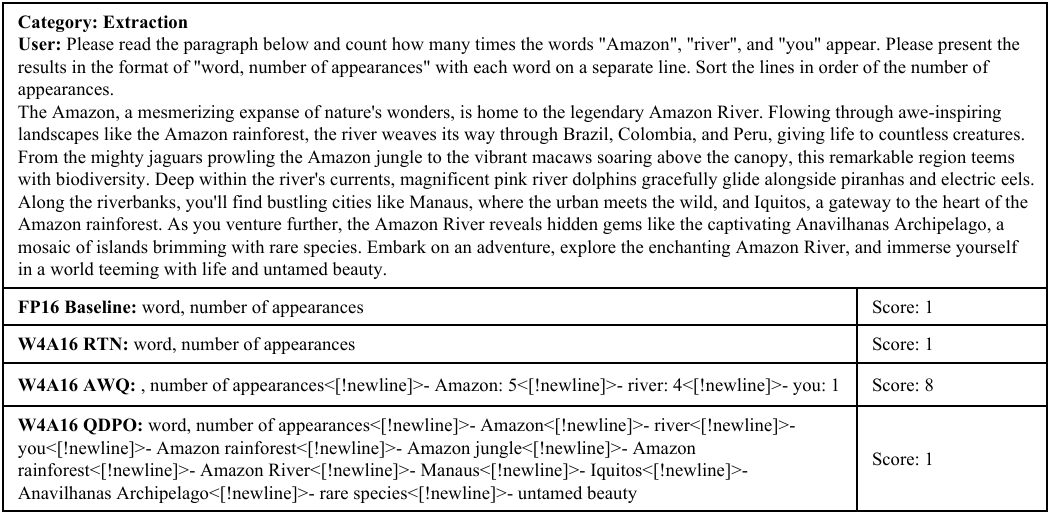}}
\caption{MT-Bench single-grading judgment example (Category: Extraction).}
\label{fig:mt_single_extraction1}
\end{figure*}
\begin{figure*}[ht]
\centerline{\includegraphics[width=\textwidth]{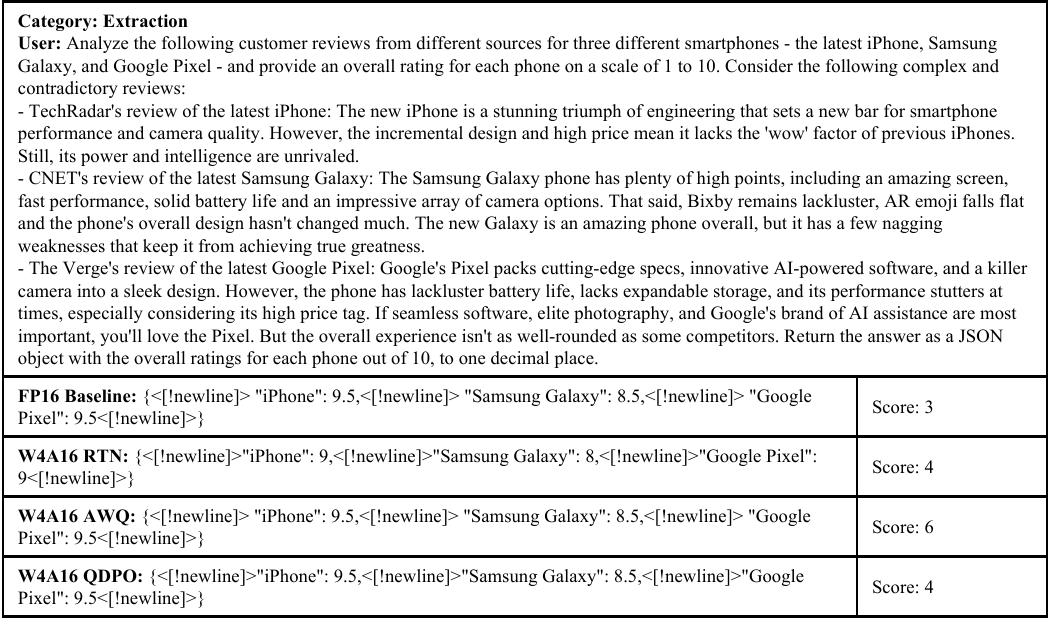}}
\caption{MT-Bench single-grading judgment example (Category: Extraction).}
\label{fig:mt_single_extraction2}
\end{figure*}

%%%%%%%%%%%%%%%%%%%%%%% FLASK example
\begin{figure*}[t]
\centerline{\includegraphics[width=\textwidth]{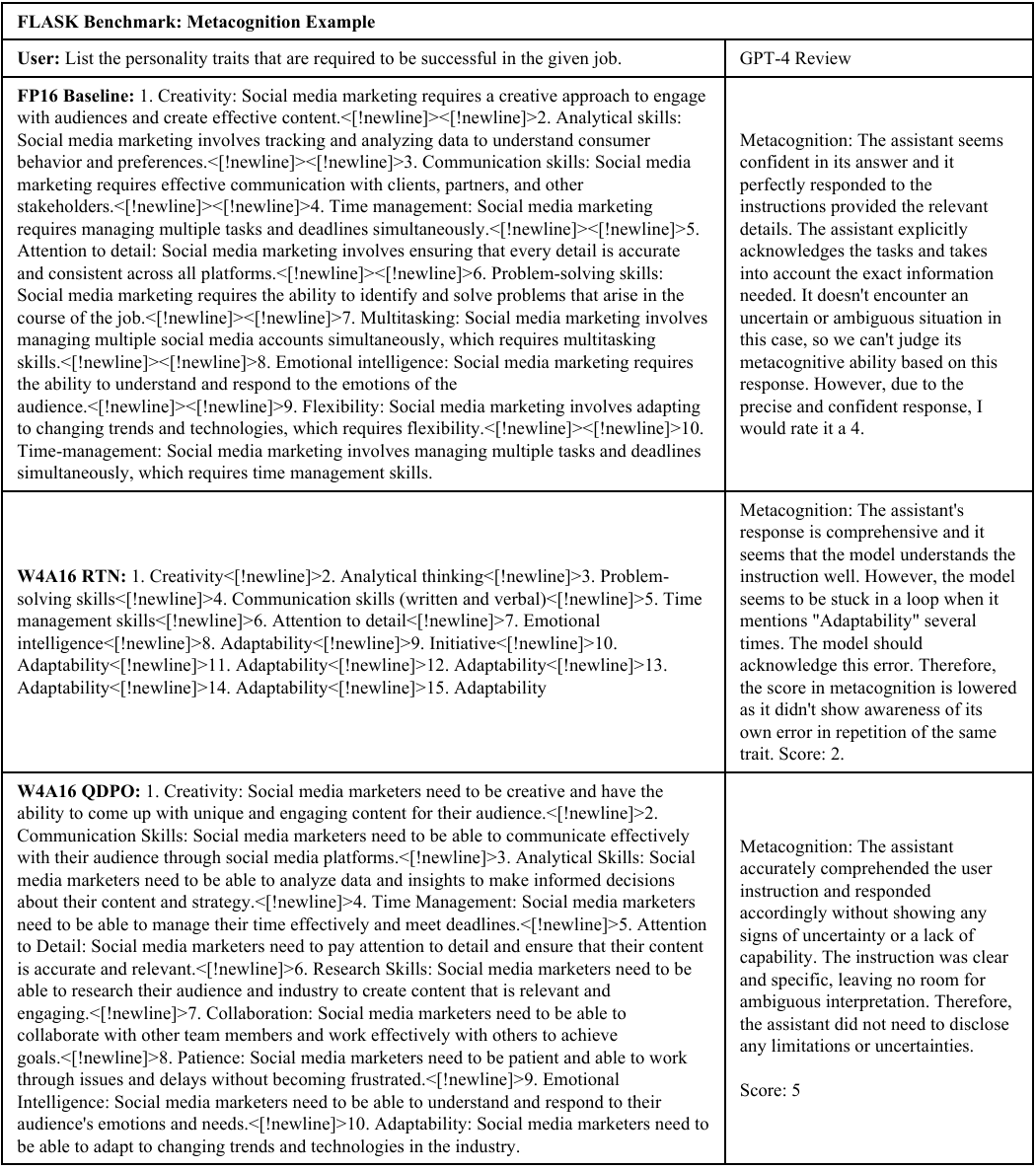}}
\caption{GPT-4 judgment in FLASK benchmark (evaluation for metacognition skill).}
\label{fig:flask_metacognition1}
\end{figure*}

\end{document}